\documentclass{article}

\usepackage{microtype}
\usepackage{graphicx}
\usepackage{subfigure}
\usepackage{amsmath}
\usepackage{amsfonts}
\usepackage{booktabs} 
\usepackage{xcolor} 
\usepackage{algorithm}
\usepackage{algcompatible}
\usepackage{wrapfig}
\usepackage[export]{adjustbox}

\usepackage{amsthm}
\usepackage{mathtools}

\usepackage{hyperref}

\newtheorem*{ass0}{Assumption 0}
\newtheorem*{ass1}{Assumption 1}
\newtheorem*{ass2}{Assumption 2}
\newtheorem*{thm41}{Theorem 4.1}
\newtheorem*{thm42}{Theorem 4.2}

\newtheorem*{rem1}{Remark 1}

\newcommand{\indep}{\perp\!\!\!\perp}

\newcommand{\N}{\mathcal{N}}
\newcommand{\D}{\mathcal{D}}
\newcommand{\G}{\mathbf{G}}
\newcommand{\lb}{\mathbf{l}}
\newcommand{\ub}{\mathbf{u}}
\newcommand{\gv}{\mathbf{g}}

\newcommand{\x}{\mathbf{x}}
\newcommand{\z}{\mathbf{z}}
\newcommand{\X}{\mathbf{X}}
\newcommand{\LL}{\mathbf{L}}
\newcommand{\Z}{\mathbf{Z}}
\newcommand{\y}{\mathbf{y}}

\newcommand{\mcs}{M}

\usepackage[main, final]{neurips_2025}

\usepackage{amsmath}
\usepackage{amssymb}
\usepackage{mathtools}
\usepackage{amsthm}

\usepackage[capitalize,noabbrev]{cleveref}

\theoremstyle{plain}
\newtheorem{theorem}{Theorem}[section]

\newtheorem{corollary}[theorem]{Corollary}
\theoremstyle{definition}

\theoremstyle{remark}

\title{Regression Trees Know Calculus}

\begin{document}

\author{Nathan Wycoff \\ 
Department of Mathematics and Statistics\\ University of Massachusetts \\
Amherst, MA 01003 \\
\texttt{nwycoff@umass.edu}
}


\maketitle

\begin{abstract}
Regression trees have emerged as a preeminent tool for solving real-world regression problems due to their ability to deal with nonlinearities, interaction effects and sharp discontinuities.
In this article, we rather study regression trees applied to well-behaved, differentiable functions, and determine the relationship between node parameters and the local gradient of the function being approximated.
We find a simple estimate of the gradient which can be efficiently computed using quantities exposed by popular tree learning libraries. 
This allows tools developed in the context of differentiable algorithms, like neural nets and Gaussian processes, to be deployed to tree-based models.
To demonstrate this, we study measures of model sensitivity defined in terms of integro-differential quantities and demonstrate how to compute them for regression trees using the proposed gradient estimates. 
Quantitative and qualitative numerical experiments reveal the capability of gradients estimated by regression trees to improve predictive analysis, solve tasks in uncertainty quantification, and provide interpretation of model behavior.
\end{abstract}

\section{Introduction}\label{sec:intro}

Tree-based methods, such as regression trees, are a workhorse of the contemporary data scientist. 
Their ease of use, computational efficiency and predictive capability without the need for extensive feature engineering makes them popular with practitioners. 
The most widely used version of regression trees approximate with greedily constructed, piecewise-constant functions than can handle data which exhibit discontinuities or divergent behavior in various parts of the feature-space. 
Perhaps because of their capability to tackle pathological problems, it seems that some of their properties in approximating well-behaved functions may have gone unnoticed.

In this article, we will study the approximation of a well-behaved differentiable function $f$ on the unit cube in dimension $P$ with a piecewise constant regression tree.
In particular, we will investigate a means of approximating $\nabla f$ using only information contained in the tree structure computable in a single pass through the tree.  
We find a simple and easily computable quantity analogous to a finite difference and 
use it to rapidly form estimates of integro-differential quantities.
Previously, gradient estimation in regression trees has been studied in the context where the leaves have differentiable models, and the gradients of these models are used to estimate the gradient (e.g. \citet{chaudhuri1995generalized,loh2011classification}). 
However, the constant-leaf tree remains prevalent in practice, and the purpose of this article is rather to examine the \textit{implicit} gradient estimation that occurs within these constant-leaf trees where, formally, the gradient of the tree is almost-everywhere zero.

With a gradient estimator in hand, we unlock for tree-based models the stable of existing gradient-based methods for variable interpretation and dimension reduction developed in other areas.
Among many possibilities, we will study in particular the Active Subspace Method \citep{constantine2015active},  a global dimension reduction technique from the Uncertainty Quantification literature, and the Integrated Gradient Method \citep{sundararajan2017axiomatic}, a local model interpretation technique from the neural network literature.

Active Subspaces provide for linear dimension reduction, which is already commonly used in the setting of regressions trees, such as when using random projection or PCA for rotated trees \citep{breiman2001random,rodriguez2006rotation}. 
In contrast to these methods, however, active subspaces consists of supervised linear dimension reduction which takes the relationship between features and response into account.  
Already, supervised linear dimension reduction has been proposed for use with tree based methods where, e.g. a kernel method is used to perform the dimension reduction which is then applied to a tree-based method  \citep{shan2015learning}. 
But here, we show how to actually use the tree itself to perform a linear sensitivity analysis, rather than relying on a helper model to do this. 
This is essential if the analyst is interested in \textit{model}-interpretation (as contrasted with \textit{data}-interpretation \citep{chen2020true}), and to the best of our knowledge the application of the Active Subspace method, enabled by our novel gradient estimates, is the first such linear sensitivity metric for trees.
And while we've so far discussed what active subspaces can do for regression trees, we don't think this new relationship will be one-sided.
Our numerical experiments show that regression trees can serve as scalable estimators of the active subspace, favorable to existing methods in certain circumstances.
We hope this can highlight the potential for regression trees in the gradient-based UQ space.

\begin{figure}
    \centering

    \includegraphics[width=0.3\linewidth]{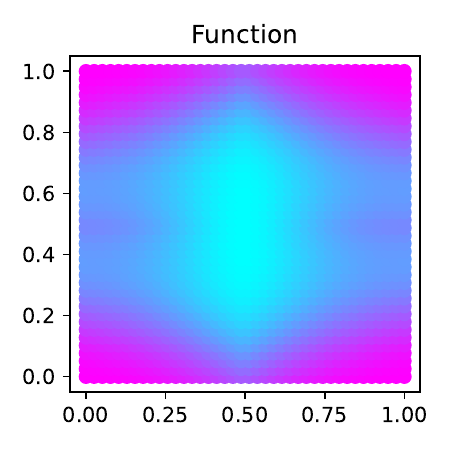}
    \includegraphics[width=0.3\linewidth]{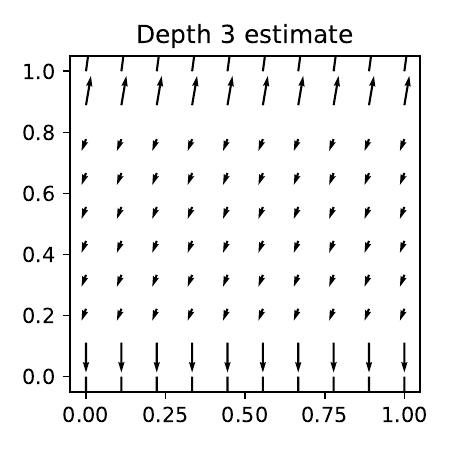}
    \includegraphics[width=0.3\linewidth]{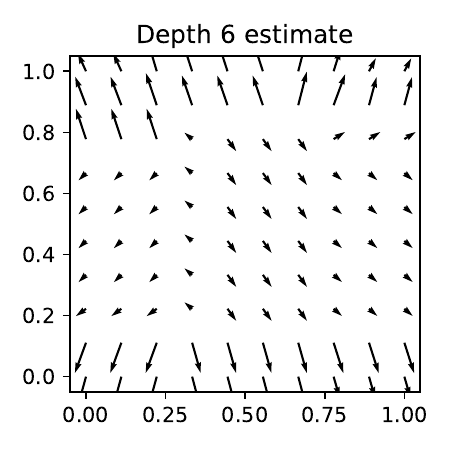}
    
    \includegraphics[width=0.3\linewidth]{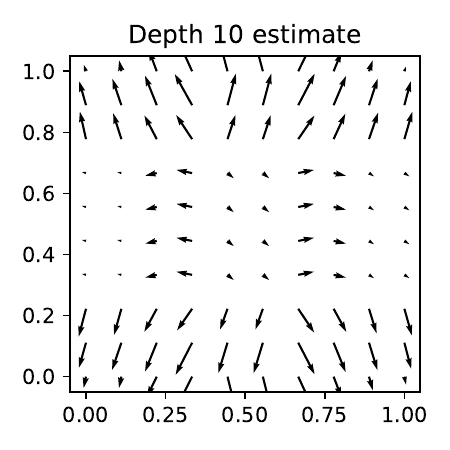}
    \includegraphics[width=0.3\linewidth]{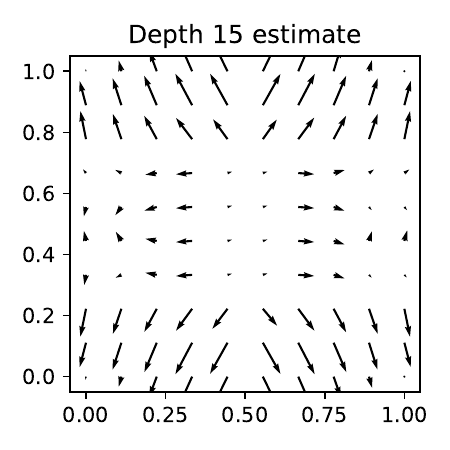}
    \includegraphics[width=0.3\linewidth]{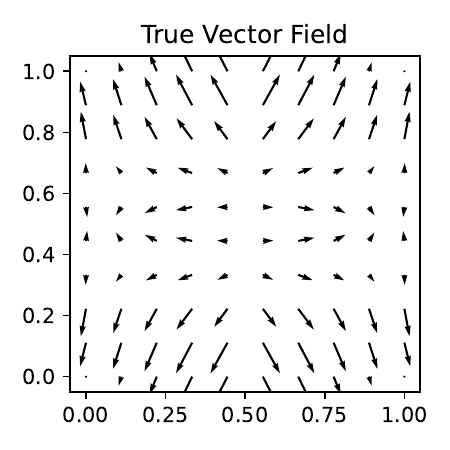}
    
    \caption{\textbf{Illustration of Gradient Estimates}. The top left gives a target function, and the bottom right gives its gradient vector field.
    Shown in between are estimates of the gradient extracted from a regression tree fit to data from the function converging to the true vector field.}
    \label{fig:hook}
\end{figure}

We view the major contributions of our article as follows:
\begin{enumerate}
	\item We study a simple algorithm to extract gradient and integrated gradient information from regression trees. 
	\item We show how to use this to port gradient-based interpretability techniques from other fields to benefit interpretability of regression trees.
	\item We find that in certain circumstances gradient-enabled regression trees can produce better estimates of active subspaces than existing UQ methods.
\end{enumerate}
We begin in Section \ref{sec:bg} by discussing pertinent background on regression trees and gradient-based interpretability.
Our main methodological contributions are given in Section \ref{sec:est} and developed theoretically in Section \ref{sec:theory}.
Subsequently, we illustrate these procedures on numerical examples in Section \ref{sec:num} before offering conclusions and future research directions in Section \ref{sec:disc}.

\section{Background}\label{sec:bg}

\begin{figure}
    \centering
      \begin{minipage}{.85\linewidth}
    \centering
    \includegraphics[width=0.85\linewidth]{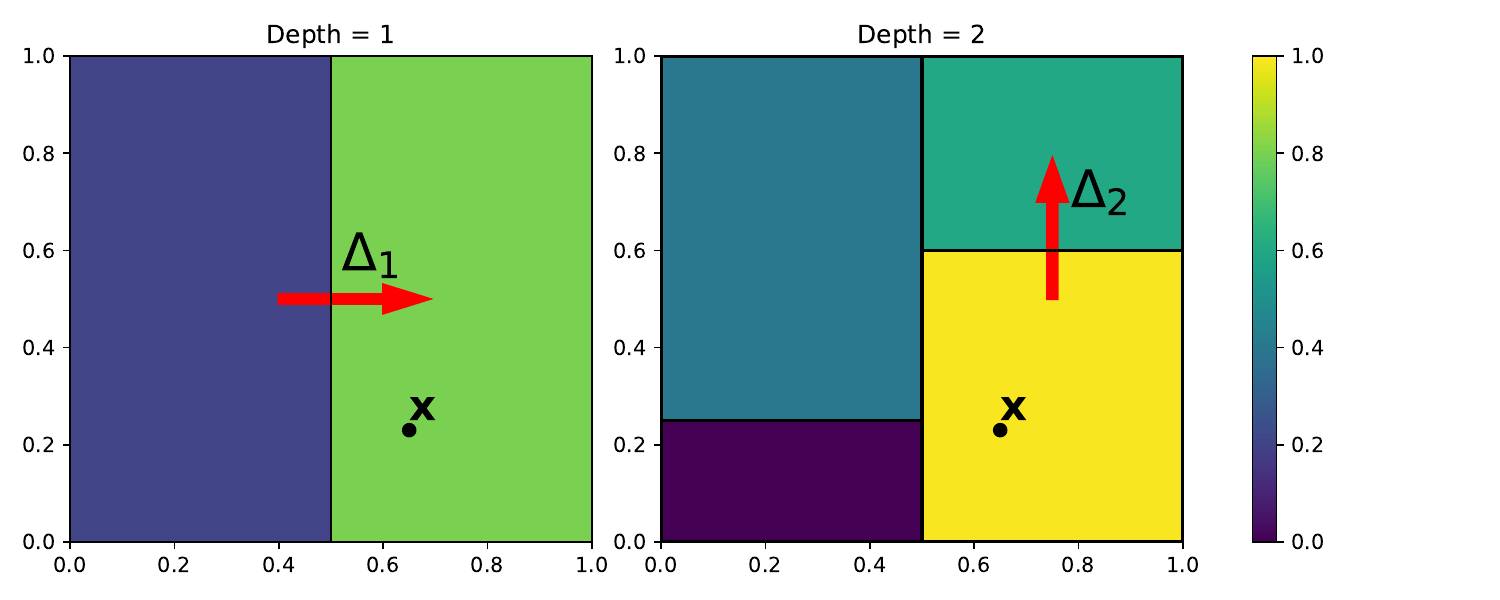}
  \end{minipage}%
  \hspace{-6em}
  \begin{minipage}{.15\linewidth}
    \begin{align*}
    \nabla f (\x) \approx 
    \begin{bmatrix}
        \Delta_1 \\
        \Delta_2 \\
    \end{bmatrix}=
    \\
    \begin{bmatrix}
        \frac{
        {\color[HTML]{7ad151}\bf 0.8}
        -
        {\color[HTML]{414487}\bf 0.2}
        }{1-0} \\
        \frac{
        {\color[HTML]{22a884}\bf 0.6}
        -
        {\color[HTML]{ced900}\bf 1.0}
        }{1-0} \\
    \end{bmatrix}
    =
    \begin{bmatrix}
        0.6\\
        -0.4
    \end{bmatrix}
    \end{align*}
  \end{minipage}
    \caption{We can extract finite difference gradient approximations from a regression tree by comparing values of adjacent nodes in splits.}
    \label{fig:main_idea}
\end{figure}

While small decision trees are easily interpretable, they become significantly less so as they grow in size or if they are aggregated via Random Forests \citep{breiman2001random,breiman2002manual} or gradient boosting \citep{friedman2001greedy}.
Consequently, developing methods for increasing the interpretability of trees is of great importance \citep{louppe2014understanding}.

\subsection{Variable Interpretation and Selection in Regression Trees}

Several methods for variable importance were suggested alongside some of the original tree-based methods \citep{breiman2017classification,breiman2001random,breiman2002manual}, the two most salient for our study being the Mean Decrease in Impurity (called MDI or TreeWeight \cite{kazemitabar2017variable}) and the Feature Permutation method.
MDI assigns importance to variables in proportion to the mean reduction in cost when splitting along a given variable, an intuitively reasonable approach. 
However, MDI as originally proposed was sensitive to properties of the feature variables as well as to the depth within the tree that a split occurred, leading to ``debiased" variants \citep{sandri2008bias,li2019debiased}.
\citet{kazemitabar2017variable} as well as \citet{klusowski2020nonparametric} studied theoretical properties of a simplified version of this metric using only the first split (``stumps").
On the other hand Feature Permutation is based on measuring decrease in performance when shuffling a given feature, and though the method does have some desirable properties \citep{ishwaran2007variable}, it has come under criticism for undesirable behavior in the context of correlated predictors \citep{hooker2021unrestricted}.
Empirical studies have found issues in both of these methods \citep{strobl2007bias,strobl2008conditional}.
Against this backdrop of negative results on the behavior of these initially suggested methods came contributions in general-purpose machine learning interpretability methods.
SHAP values \citep{lundberg2017unified} have been proposed as a general-purpose tool for understanding machine learning models.
However, they would prove especially popular for understanding tree-based methods where efficient algorithms have been proposed for estimating these otherwise combinatorially difficult quantities \citep{lundberg2019explainable,karczmarz2022improved}.

%
%

\subsection{Gradient-based Model Interpretation}\label{sec:bg_grads}

The gradient of a function tells us how its output is related to its input over short distances, and is a natural candidate to help explain the behavior of a model. 
\citet{hechtlinger2016interpretation} suggested to simply look at the gradient of a model evaluated at a particular observation to gain local understanding.
Another approach is the Integrated Gradient \citep{sundararajan2017axiomatic}, introduced in the context of neural networks, which involves privileging some setting of input features which is called the \textit{reference point}, which we'll denote $\x^*$.
Subsequently, in order to explain a prediction at a given point $\x$, we integrate the gradient along the path from $\x^*$ to $\x$ and multiply elementwise against that difference; that is, denoting the output of the neural network as $f$:  
\begin{equation}
    IG(\x) := (\x-\x^*) \odot \int_{\alpha=0}^1 \nabla f(\alpha\x + (1-\alpha)\x^*) d\alpha \, .
\end{equation}
This theoretically appealing approach has also seen practical success in applications including medicine \citep{sayres2019using} and chemistry \citep{mccloskey2019using}.

One could also integrate the gradient over a larger part of the space to perform a more global sensitivity analysis.
This is the idea behind the Active Subspace Method
\citep{constantine2015active}.
In our context, this involves defining the Active Subspace Matrix as:
\begin{equation}
    \mathbf{C}^f_{\mu} := \int_{[0,1]^P} 
    \nabla f(\x) \nabla f(\x)^\top d\mu(\x) \,,
\end{equation}
where $\mu$ is a measure.
Eigenanalysis on $\mathbf{C}^f_\mu$ reveals linear combinations of features which are important, similar to PCA.
Various techniques have been developed to estimate active subspaces in the presence only of input-output data. 
A Polynomial Ridge Approximation (PRA) procedure \citep{hokanson2018data} is available, as is a closed form estimate using Gaussian processes \citet[GP;][]{wycoff2021sequential}, and also an approach based on neural networks called the Deep Active Subspace Method \citep[DASM;][]{tripathy2019deep,edeling2023deep} has been proposed.

\subsection{Variable Importance versus Sensitivity Analysis}

Though originating in different academic communities, the ideas of Variable Importance (from the Machine Learning community) and Sensitivity Analysis (from the Uncertainty Quantification community) are closely related. 
See \citet{antoniadis2021random} for a recent review of using regression trees for (non-gradient-based) sensitivity analysis.
\citet{vannucci2024enhancing} explicitly port ideas from the sensitivity analysis literature to enhance variable importance in tree-based methods.
\citet{elie2024random} use regression trees to estimate a form of sensitivity analysis based on quantiles.

\section{Integro-Differentiation of Regression Trees}
\label{sec:est}

\begin{wrapfigure}{r}{0.4\textwidth}
    \vspace{-1.5em}
    \centering
    \includegraphics[width=0.98\linewidth]{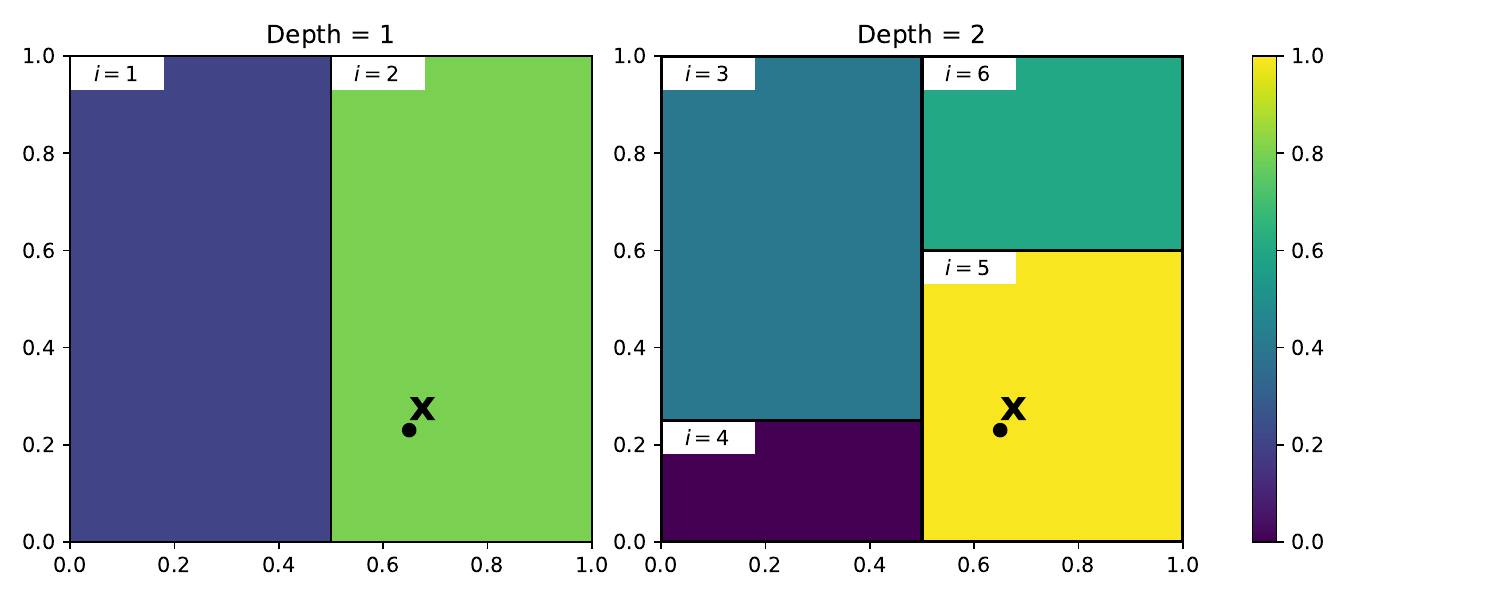}
    \caption{ \textbf{Illustration of Notation.} In the same example tree as Figure \ref{fig:main_idea} with a depth $K=2$, examples of our notation is as follows: 
     the indices of nodes at each depth are $\D_1 = \{1,2\}$ and $\D_2 = \{3,4,5,6\}$;
     the children of node $2$ are  $c^2_l = 5$ and $c^2_r = 6$;
     conversely the parent of node $5$ is given by $\rho_5 = 2$;
     the bounds of node $5$ are $\lb^5=[0.5,0]$ and $\ub^5 = [1,0.6]$;
     the value of intermediate node $2$ is $v_2 = 0.8$ and the value of leaf node $6$ is 0.6;
     since the ``root node" $0$ is split along the x-axis, $\sigma_0=1$ and since nodes $1$ and $2$ are split along the y-axis, $\sigma_1 = \sigma_2 = 2$;
     since the point $\mathbf{x}$ lies within the nodes $2$ and $5$ at depths 1 and 2 respectively, we have that $B^1(\x) = 2$ and $B^2(\x)=5$.}
    \label{fig:notation}
    \vspace{-0.5em}
\end{wrapfigure}

In this section we propose an easily computed estimator of derivatives and integrals of derivatives for regression trees.
While the underlying idea is simple, discussing it rigorously requires a fair amount of notation; Figure \ref{fig:notation} illustrates our notation on a simple regression tree.
Say we fit a regression tree of depth $K$ to data $(\x_n,y_n=f(\x_n)+\epsilon_n)$ with $\x_n\in[0,1]^P,y_n\in\mathbb{R}$ for $n\in\{1,\ldots,N\}$, and
$\epsilon_n$ mean zero with finite variance.
Purely for simplicity, we will assume that the tree has leaf nodes only at depth $K$, i.e. there is no path through the tree that ends before depth $K$.
Associate with every node from all depths an integer index $i$ such that the root node is labeled $0$ and if node $i$ is deeper in the tree than node $j$, then $i>j$.
That is, the indices at depth 1 are given by $\{1,2\}$, at depth 2 by $\{3,4,5,6\}$, and generally at depth $k$ the indices are given by $\N_k = \{2^{k-1}+1,\ldots, 2^k\}$.
For any depth $k$, we define the function $B^k:[0,1]^P\to\N_k$ where $B^k(\x)$ gives the index of the node at depth $k$ which contains $\x$.
We denote the variable along which the $i^\textrm{th}$ node splits by $\sigma_i\in\{1,\ldots,P\}$ and the threshold at which that split occurs as $\tau^i$.
Those points in node $i$ for which  $x_{\sigma_i}\leq \tau^i$ are mapped to the left child node of $i$, denoted $c^i_l$, while those for which $x_{\sigma_i} > \tau^i$ are mapped to the right child node, denoted $c^i_r$.
For each node $i$ with child nodes, we denote the mean of the response variable in the left child node by $\hat\mu^i_l$ and in the right child node by $\hat\mu^i_r$.
That is, $\hat\mu^i_l$ gives the average of $y_n$ for all $n\in\{1, \ldots N\}$ such that $B^{D_i+1}(\x_n) = c_l^i$ where $D_i$ is the depth of node $i$ (and so $D_i+1$ is the depth of its child nodes).

\subsection{Differentiation}

The difference between $\hat\mu_l^i$ and $\hat\mu_r^i$, which contain the mean of the data on either side of the threshold for node $i$,
contains information about the $\sigma_i^\textrm{th}$ component of the gradient in that area.
By dividing this difference proportional to the size of the node along dimension $\sigma_i$, we form a quantity similar to a finite difference.
Let $[l^i_p,u^i_p]$ denote the extent of node $i$ along dimension $p$ for $p\in\{1, \ldots, P\}$. 
Then define the following quantity:
\begin{equation}
	\gamma_i
    := \frac{2(\hat\mu_r^i-\hat\mu_l^i)}
    {u^i_{\sigma_i}-l^i_{\sigma_i}}
    \approx
    \frac{\partial f}{\partial x_{\sigma_i}}(\x) \,\, \forall \x \in [\lb^i,\ub^i]
    \, .
\end{equation}
We have used the notation $\lb^i = [l^i_1,\ldots, l^i_P]$ to denote the vector of lower bounds of the extent of node $i$ and similar for $\ub^i$, and by $[\lb^i,\ub^i]$ we denote the set of points lying within the bounds of the $i^\textrm{th}$ node.
The ratio $\gamma_i $ contains information about the partial derivative of $f$ with respect to $x_{\sigma_i}$ (which is, again, the variable that node $i$ splits along) inside of $[\lb^i,\ub^i]$.
The idea that 
$\gamma_i
\approx
\frac{\partial f(\x)}{\partial x_{\sigma_i}} \big|_i $
for any $\x$ in node $i$ only makes sense if the gradient does not vary much within it.
We expect this to be the case only for very small nodes, or in other words, very deep trees, estimating functions with gradients which vary smoothly; we make this precise in the next section.

At a particular node, we can only form estimates of a single partial derivative.
However, if our tree is sufficiently deep, we can combine estimates of partial derivatives from nodes at multiple depths to form an estimate of the entire gradient.
Recall that $\N_k=\{2^{k-1}+1, \ldots, 2^{k}\}$ contains the index of all nodes of depth $k$ and let $\rho_i$ denote the index of the parent of node $i$ (which will be a member of $\N_{k-1}$).
Then Algorithm \ref{alg:TBDA} shows how to aggregate the partial derivative estimates into a gradient estimate.

\begin{algorithm}
\caption{Tree-Based Gradient Estimation}\label{alg:TBDA}
\begin{algorithmic}
 \STATE $\G^{i} \gets \mathbf{0} \in \mathbb{R}^P$ for all $i$. 
\FOR{$k \in \{1,\ldots,K\}$} 
    \FOR{$i \in \N_k$} 
        \STATE $\G^i \gets \G^{\rho_i}$ \COMMENT{Get parent's estimate} 
        \STATE $\G^{i}[\sigma_i] \gets \frac{2(\hat\mu^i_r-\hat\mu^i_l)}{u^i_{\sigma_i}-l^i_{\sigma_i}}$ \COMMENT{Update along split direction}
    \ENDFOR
\ENDFOR
\end{algorithmic}
\end{algorithm}

We thus have associated each node $i$ with an estimator of the gradient, though if we have not split along a variable $p$ in the tree upstream of $i$, the estimate of that component will simply be $0$.
This is actually somewhat reasonable, as the tree will split along variables with small gradient components less frequently.

Of course, nothing would stop us from comparing values of means that are adjacent in the input domain but not adjacent in the tree structure 
(for example, nodes 3 and 6 of Figure \ref{fig:notation}).
However, the advantage of considering only nodes adjacent in the tree is that this Tree-Based Gradient Estimator (TBGE) can be computed efficiently merely by traversing the tree.
Furthermore, if we needed only an estimate of a gradient at a particular point $\x$, we would need only to traverse the regression tree in the standard manner, visiting only nodes upstream of the leaf node $\x$.
To be explicit, if the tree is of depth $K$, recall that $B^K(\x)$ gives the index of the node of the greatest depth which contains $\x$.
We use as our estimate of $\nabla f(\x)$ the TBGE associated that node; that is, $\tilde{\nabla} f(\x) := \G^{B^K(\x)}$.


\subsection{Integration}

We next develop estimators for quantities of the form:
\begin{equation}
\label{eq:If}
    \mathcal{I}(f) = \int h(\nabla f(\x)) d\mu(\x) \,.
\end{equation}
Here, $h$ is some integrable $\mathbb{R}^P\to\mathcal{H}$ function and $\mu$ is a measure on $\mathbb{R}^P$. 
The Integrated Gradient at point $\mathbf{z}$ with reference point $\mathbf{z}^*$ may be written in this form
by choosing $h(\mathbf{a}) = (\mathbf{z}-\mathbf{z}^*)\odot \mathbf{a}$, $\mathcal{H}=\mathbb{R}^P$, and $\mu$ as the degenerate uniform measure on the line segment between $\mathbf{z}$ and $\mathbf{z}^*$.
Similarly, the Active Subspace matrix can also be seen to belong to this class by 
setting $h(\mathbf{a}) = \mathbf{a}\mathbf{a}^\top$, $\mathcal{H}=\mathbb{R}^{P\times P}$ and arbitrary $\mu$.

We will consider two classes of estimators for such quantities.
The first is a Monte Carlo Estimator (MCE) which is appropriate whenever $\mu$ is a probability measure which it is possible to sample from.
The second is a Partition-Based Estimator (PBE) which is suitable whenever we can compute $\mu([\mathbf{a},\mathbf{b}])$ for arbitrary $\mathbf{a}$, $\mathbf{b}$.

Start with the MCE.
We fix some Monte Carlo sample size $\mcs$ and form the standard Monte Carlo approximation, then plug in the TBGE where the gradient appears:
\begin{equation}
    \mathcal{I}(f)\approx 
    \frac{1}{\mcs}
    \sum_{\x_m\sim\mu}
    h(\nabla f(\x_m)) 
    \approx
    \frac{1}{\mcs}
    \sum_{\x_m\sim\mu}
    h(\tilde{\nabla} f(\x_m)) 
    \,.
\end{equation}
We denote this estimator by $\hat{\mathcal{I}}_{MC} (f)$.
In particular, we can form a Tree-Based Integrated Gradient (TBIG) as follows by sampling random uniform variables $u_m$ on $[0,1]$ and then forming:
\begin{equation}
\label{eq:ig}
    \hat{IG}(\x)
    =
    (\x-\x^*)\odot \frac1{M}
    \sum_{m=1}^M  \tilde{\nabla} f(u_m\x + (1-u_m)\x^*) \,.
\end{equation}

Next, we consider the PBE.
This involves simply computing a weighted average of the action of $h$ on the TBGE on each node at the penultimate depth $K-1$, weighted by the measure of that node.
For simplicity, we assume that $\mu([0,1]^P)=1$.
In notation, we define the estimator as
\begin{equation}
    \hat{\mathcal{I}}_{PE}(f) = 
    \sum_{i\in\N_{K-1}} 
    h(\G^i)\mu([\lb^i,\ub^i])
    \,.
\end{equation}
We can use this to form Tree-Based Active Subspace (TBAS) for $f$ as follows:
\begin{equation}
    \hat{C}^f_\mu 
    = 
    \sum_{i\in\N_{K-1}} 
    \G^i \G^{i\top} \mu([\lb^i,\ub^i]) \,.
\end{equation}

Which estimator is preferred in practice will depend on which properties of $\mu$ are computationally tractable.
The ability to sample from $\mu$ makes the MCE tractable while the ability to compute the measure of an arbitrary rectangle makes the PBE tractable.
If both are available, the PBE will be favorable as it does not have Monte Carlo error.

\section{Asymptotic Analysis}
\label{sec:theory}

We'll now make rigorous the intuition developed in the previous section with a basic asymptotic analysis of the TBGE under simplifying assumptions.
In particular, we assume that the split decisions are made on the basis of the feature variables only without reference to the outcome variables. 
Though this does not accommodate most algorithms used in practice for supervised learning, such as the the CART algorithm, these assumptions nevertheless are sufficient to develop some intuition as to qualitative properties of the approximation.
Additionally, similar simplifying assumptions have often been used in recent articles
\citep[e.g.][]{gao2022towards,ma2023decision,cai2023extrapolated, wen2022random,klusowski2021sharp}.

\begin{theorem}\label{thm:grad_conv}
Let $S(N)$ denote the number of splits in the tree as a function of the sample size.
Let 
$
\frac{\tilde{\partial}^N f(\x)}{\tilde{\partial} x_{p}} 
$
denote the TBGE at point $\x\in[0,1]^P$ with sample size $N$.
Under Assumptions 0, 1, and 2 of the Appendix, we have that:
\begin{equation*}
    \frac{\tilde{\partial}^N f(\x)}{\tilde{\partial} x_{p}} 
    = 
    \frac{\partial f(\x)}{\partial x_{p}}
    +
    O_P\left(
    \frac{2^{\frac{P+2}{2P}S(N)}}{\sqrt{N}}
    +
    P2^{-\frac{S(N)}{P}}
    \right)
    \,.
\end{equation*}
\end{theorem}

In order for that error term to vanish, we need to choose $S(N)$ growing even slower than $\log N$, such as $P\log\log N$ (see Appendix for more discussion).
This choice yields an error rate of the gradient estimates of $O_P(P(\log N)^{-1})$, which is slow relative to typical estimation rates.
This suggests that this gradient estimator may be most useful in situations where a coarse estimate is acceptable.
The reason that this rate needs to be so slow is that on the one hand, the number of datapoints in each cell needs to diverge in order for the statistical error to vanish. 
But on the other hand, the Taylor series approximation error in the finite difference analog requires for the size of the cells to vanish, which implies an exploding number of cells. 
These factors combine to require a large number of data to achieve small errors.
However, they leave open the possibility of leveraging the computational efficiency of regression trees to quickly create approximate gradient estimates in large samples, as our numerical experiments will show.


In the remainder of this section, we will establish consistency of estimates for integro-differential quantities, which will require only consistency of the TBGE, which we have just established.
We begin with the PBE; the remaining results are in the Appendix.

\begin{theorem}
\label{thm:pbe}
Let $h:\mathbb{R}^P\to\mathbb{R}^P$ be bounded in the sense that $h(\z) \leq C_h \Vert\z\Vert^{a_h}$ for $a_h>0$.
Under Assumptions 0, 1 and 2 of the appendix, the Partition-Based estimator converges to the true integro-differential quantity as the sample size diverges. That is,
\begin{align*}
    \lim_{N\to\infty} \hat{\mathcal{I}}_{PB}(f) = \mathcal{I}(f) \,.
\end{align*}
where $\mathcal{I}(f)$ is the integro-differential quantity given in Equation \ref{eq:If}.
\end{theorem}
Theorem \ref{thm:mce} of the Appendix gives a similar result for the MCE.
We also explicitly state there the implications of these results for Active Subspace and Integrated Gradient estimation, namely the consistency of the estimators $\hat{C}^f_\mu$  and $\hat{IG}$ proposed in Section \ref{sec:est}.

We have thus established consistency of gradient-based model interpretation quantities for regression trees. 
Next, we will empirically study their behavior in finite samples via a battery of numerical experiments.

\section{Numerical Experiments}\label{sec:num}

We now study how the proposed gradient estimator might be profitably exploited in practice. 
We begin with a qualitative study showing how a tree-based integrated gradient (TBIG) can facilitate local model interpretation.
Then come three quantitative studies, first investigating the potential of a Tree-Based Active Subspace (TBAS) to improve prediction accuracy of a downstream tree via a rotation of the space. 
Subsequently, we evaluate the capacity of regression trees to estimate the Active Subspace in low and high dimension.
Finally, we end with another qualitative study demonstrating how a TBAS can provide data visualization.

\subsection{Integrated Gradient for Tree-Based Methods}

\begin{figure}[h]
    \centering
    \begin{tabular}{c|c}
         \includegraphics[width=0.48\textwidth]{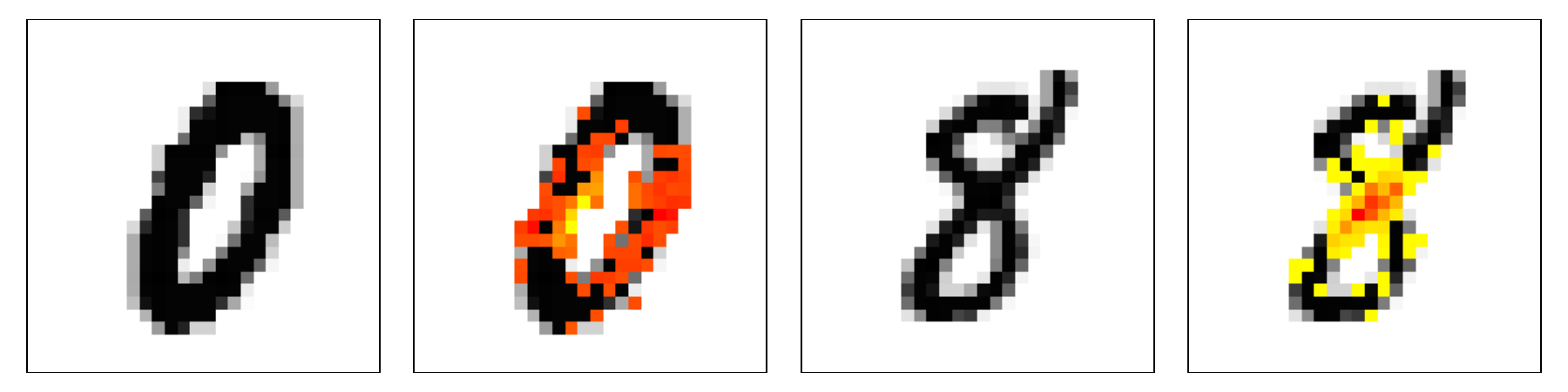}
         &  
         \includegraphics[width=0.48\textwidth]{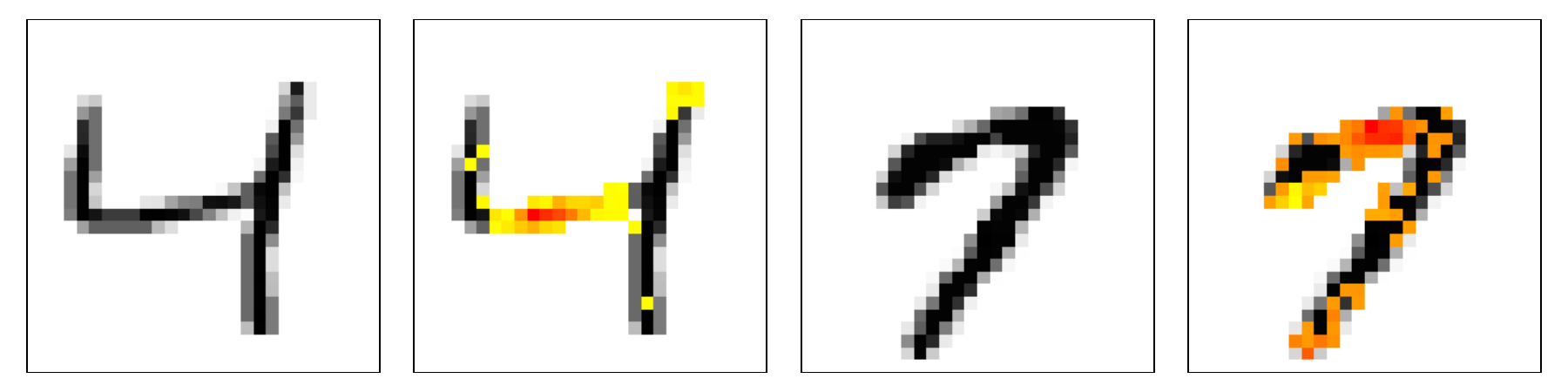}
    \end{tabular}
    \caption{\textbf{Integrated Gradient for Trees.} Each pair of panels gives a training example from MNIST. The second pair in the image superimposes the IG values onto the example. Redder means more strongly suggesting correct class membership.}
    \label{fig:ig}
\end{figure}

We now consider a random forest classifier fit to two subsets of the MNIST dataset, one contrasting zeros against eights and the other fours against sevens.
While classification was not the focus of this article, we present some preliminary numerical experiments in Appendix \ref{app:class}.
We use the estimator $\hat{IG}$ of Equation \ref{eq:ig} with $M=500$ random points along a given path.
Figure \ref{fig:ig} gives the results of this analysis.
We see that when comparing eights against zeros, the intersection point of the eight is most important, while the left and right sides of the zero are.
This makes sense as they are the parts of each digit which are distinct.
By contrast, the shared upper and lower arcs are not highlighted.
In the four versus seven case, the location of the horizontal bar seems to be most influential in determining one versus the other. 
This suggests that the model is comparing the relative position of the bar to the rest of the digit in making its classifications.

\subsection{Active-subspace Rotated Trees for Predictive Analysis}\label{sec:num_pred}

\begin{table*}
\centering
\begin{tabular}{lllllllll}
\toprule
\textbf{Dataset:} & bike & concrete & gas & grid & keggu & kin40k & obesity & supercond \\
\midrule
\multicolumn{9}{c}{\textbf{Regression Tree (Depth 4)}}\\
\midrule
TBAS & \textbf{0.635} & \textbf{0.47} & \textbf{0.578} & \textbf{0.688} & \textbf{0.194} & \textbf{0.856} & \textbf{0.128} & \textbf{0.511} \\
Id & \textbf{0.655} & 0.537 & 0.595 & 0.773 & 0.35 & 0.963 & 0.226 & \textbf{0.514} \\
PCA & \textbf{0.655} & 0.543 & \textbf{0.591} & 0.773 & 0.349 & 0.964 & 0.226 & \textbf{0.513} \\
Rand & 0.656 & 0.524 & 0.593 & 0.752 & 0.349 & 0.95 & 0.226 & \textbf{0.513} \\
\midrule
\multicolumn{9}{c}{\textbf{Regression Tree (Depth 8)}}\\
\midrule
TBAS & \textbf{0.402} & \textbf{0.35} & \textbf{0.429} & \textbf{0.521} & \textbf{0.078} & \textbf{0.586} & \textbf{0.094} & \textbf{0.392} \\
Id & \textbf{0.405} & 0.403 & \textbf{0.432} & \textbf{0.522} & 0.121 & 0.862 & \textbf{0.103} & \textbf{0.395} \\
PCA & \textbf{0.407} & \textbf{0.395} & \textbf{0.423} & \textbf{0.524} & 0.122 & 0.872 & \textbf{0.107} & \textbf{0.392} \\
Rand & \textbf{0.411} & \textbf{0.391} & \textbf{0.435} & 0.55 & 0.124 & 0.836 & \textbf{0.109} & \textbf{0.397} \\
\midrule
\multicolumn{9}{c}{\textbf{Random Forest (Depth 4)}}\\
\midrule
TBAS & \textbf{0.609} & \textbf{0.406} & \textbf{0.559} & \textbf{0.602} & \textbf{0.161} & \textbf{0.802} & \textbf{0.123} & \textbf{0.479} \\
Id & 0.65 & 0.462 & \textbf{0.574} & 0.659 & 0.344 & 0.954 & 0.173 & \textbf{0.485} \\
PCA & 0.649 & 0.462 & \textbf{0.567} & 0.654 & 0.344 & 0.954 & 0.174 & \textbf{0.486} \\
Rand & 0.646 & 0.458 & \textbf{0.57} & 0.66 & 0.33 & 0.938 & 0.173 & \textbf{0.486} \\
\bottomrule
\end{tabular}
\vspace{0.3em}
\caption{\textbf{Predictive Impact of Various Transformations on Selected Datasets}. Numbers give 100-fold RMSE; bold indicates confidence interval overlaps with the lowest confidence interval.} 
\label{tab:rot}
\end{table*}

This section evaluates the ability of TBAS to improve the accuracy of a downstream predictive analysis.
Given some mapping $\LL$, we refer to postmultiplication of the feature matrix $\X$ to form a new feature matrix $\Z:=\X\LL$ as a rotation.
In standard tree-based methods, non-diagonal rotations can have an impact on predictive performance because the axes along which splits are made have been changed.
In order to quantitatively assess the utility of TBAS, we compare the prediction error of regression trees and random forests on eight benchmark datasets fit on data augmented by a rotation of the space. 
A TBAS rotation is conducted by choosing $\LL$ to be a matrix square root of the active subspace matrix estimate, following \cite{wycoff2022sensitivity}.
We compare TBAS rotations to the those made by drawing orthogonal directions uniformly over the Grassmannian (Rand; similar to \citet{breiman2001random}) as well as those formed from a Principal Components Analysis (PCA; similar to \citet{rodriguez2006rotation}) and no rotation at all (denoted Id).
Table \ref{tab:rot} presents the results of this analysis. 
TBAS does at least as well as the other methods, and sometimes offers significant improvement (such as on the Kin40k and Keggu datasets).

\subsection{Computationally Efficient Active Subspace Estimation in Low Dimension}\label{sec:num_smol_active}

\begin{figure}
    \centering
    \includegraphics[width=0.8\textwidth]{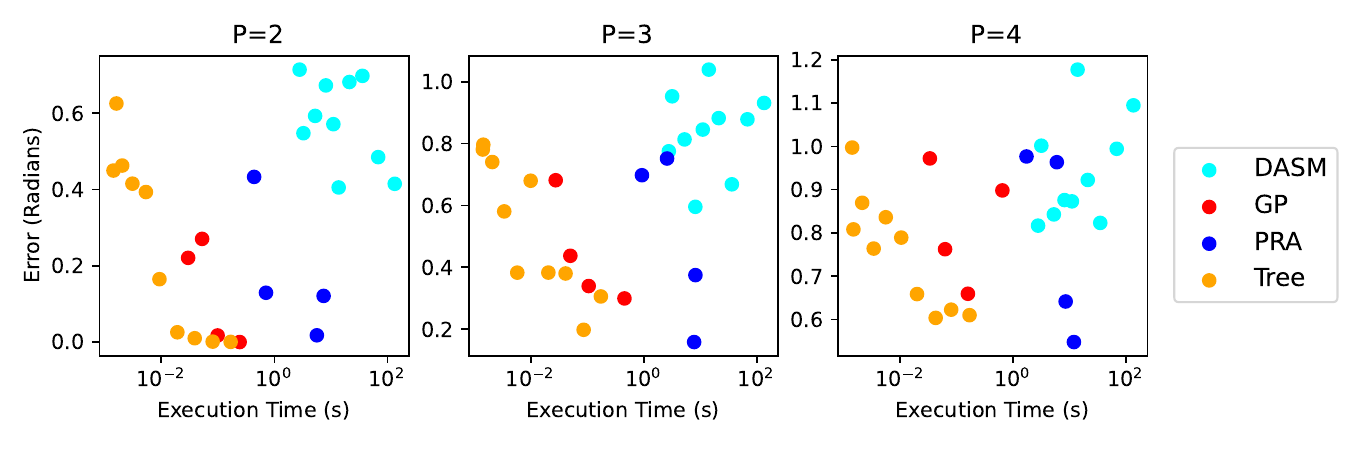}
    \caption{\textbf{Active Subspace Estimation in Low Dimension.} Execution time (x-axis) and Subspace Estimation Error (y-axis) for the four methods, lower is better.}
    \label{fig:smol_active}
\end{figure}
In this section we demonstrate the capability of TBAS to offer extremely fast, gradient-free estimation of the active subspace in low dimension when significant data are available. 
We compare to the three popular methods for gradient-free active subspace estimation introduced in Section \ref{sec:bg_grads}, based on a Gaussian Process (GP), Ridge Polynomial (PRA), and Neural Network (DASM).
Each of these methods was tasked to estimate a one dimensional active subspace in dimensions 2, 3 and 4, using a logarithmically spaced grid of sample sizes ranging from 10 to 10,000 on a toy algebraic function.
We found that the GP and PRA methods worked efficiently when sample sizes were small, but computation time grew quickly, such that we were only able to run these methods for samples of size 150 or less.
Figure \ref{fig:smol_active} shows the active subspace error against the execution time.
We see that the tree-based estimator forms the majority of the Pareto front in all three cases.

\subsection{Estimating Sparse Active Subspaces in High Dimension}


Imposing sparsity in the coefficients of linear dimension reduction can improve interpretability \citep{zou2006sparse}.
Because of the manner that tree-based methods produce gradient estimates, we conjecture that TBAS has built-in inductive bias to favor entry-wise active subspaces.
In this section, we will investigate this possibility by comparing TBAS against the DASM in estimating a sparse active subspace in dimensions 10, 50 and 100.
We will use the same setup as the previous section, except that the true subspace has only three nonzero coordinates.
Because of the sample size that is required to estimate an active subspace in high dimension, it is computationally infeasible to deploy the GP or PRA active subspace estimation methods.
Figure \ref{fig:big_active} shows the results of this analysis. 
The TBAS estimator is able to achieve much better accuracy with a significantly smaller cost.

\begin{figure*}
    \centering
    \includegraphics[width=0.99\textwidth]{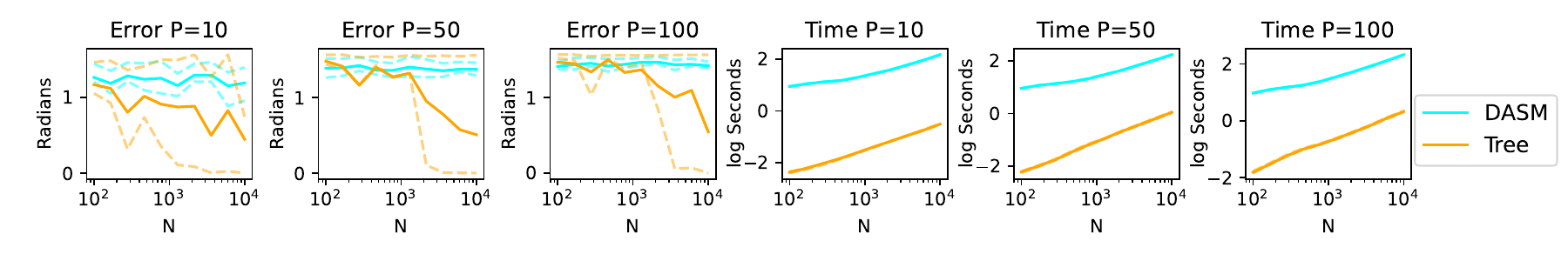}
    \caption{\textbf{Sparse Active Subspace Estimation in High Dimension.}
    Lower is better.
    }
    \label{fig:big_active}
\end{figure*}

\subsection{Dimension Reduction with the Active Subspace}

We now turn to the qualitative benefits of performing an active subspace analysis using TBAS.
We used the NHEFS dataset of biochemistry tape and mortality data which were studied by \citet{lundberg2019explainable}.
This consisted of a dataset of 14,407 observations and 90 complete variables.
We fit a TBAS to this dataset using a regression tree of depth 15, requiring at least 10 samples per leaf.
We subsequently performed an eigendecomposition of the estimated active subspace matrix.
Like \citet{lundberg2019explainable}, we found that age was the most important variable and it mapped cleanly onto the first active subspace dimension.
However, we found that the next most important dimension consisted of many variables (see Appendix \ref{app:mortality}), such that they would have been difficult to capture with a typical variable-importance analysis.
Figure \ref{fig:eig} shows the result of this active subspace analysis, which reveals that after the second eigenvalue there appears to be a gap in the spectrum of the active subspace matrix, indicating the presence of an active subspace of dimension two (left panel).
The middle panel shows a projection of the data, which breaks into two clusters along the second principal component, while the right panel shows a projection of randomly sampled points to visualize the predictive surface of the function.
The right panel reveals the predictive surface has a fairly simple, ``S-shaped" form over these two dimension.
It would not be possible to detect this by considering only main effects or two factor interactions.

\begin{figure*}
    \centering
    \includegraphics[width=0.95\textwidth]{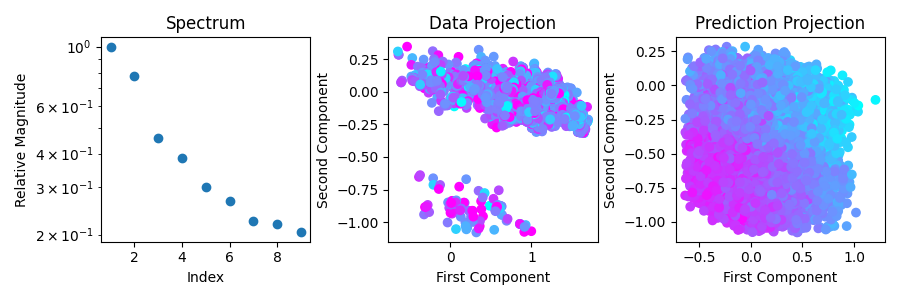}
    \caption{
    \textbf{Projection with TBAS.}
    \textit{Left:} Spectrum of active subspace matrix suggests a 2D active subspace. 
    \textit{Middle:} Projection of data onto active subspace.
    \textit{Right:} Projection of predictive surface.
    }
    \label{fig:eig}
\end{figure*}

\section{Discussion}\label{sec:disc}

\textbf{Summary:} We studied a simple method for estimating gradients in sufficiently deep regression trees on large datasets. 
Subsequently, we demonstrate how these estimates could be used to calculate active subspaces and integrated gradients.
We found significant improvement in predictive performance could be achieved by using the tree-based active subspace to rotate the space, and also found that these estimates executed quicker than existing gradient-free active subspace estimators and had useful inductive bias towards detecting coordinate-sparse subspaces. 
Finally, we used these gradient estimates to reveal the simple global structure of a tree-based model fit to a complex dataset as well as the mechanism by which a random forest conducted MNIST classification.

\textbf{Limitations:} When it comes to limitations of our experiments, when comparing our active subspace estimates to the DASM, we used a particular neural network architecture and optimizer. 
It is possible that a different configuration would have performed better. 
Additionally, when it comes to limitations of the method we studied, our analysis suggests that it would take a very deep regression tree to estimate gradients to high accuracy in high dimension.

\textbf{Conclusions:} We think this work offers two high level conclusions. 
First, that trees may have a place in the field of Uncertainty Quantification, which has more commonly used differentiable but slow to estimate surrogates such as Gaussian processes and neural networks. 
Secondly, we hope that this work can also enable improve cross-pollination between developments in interpretability for neural networks and for regression trees by creating analogs for gradient-based neural network techniques.

We also would like to comment on how TBAS fit into the existing literature on interpretability in trees. 
It is somewhat distinct from the axiomatic approach to interpretability proffered by SHAP \citep{lundberg2017unified} and Integrated Gradients \citep{sundararajan2017axiomatic}. 
While an axiomatic approach can be useful, its universality should not be exaggerated. 
\citet{hancox2021epistemic} write about how it is unlikely that any variable selection method could satisfy all possible demands, no matter how attractive its axiomatic foundation.
Since the active subspace is parameterized by a measure, it actually consists of an entire family of analyses, with different choices of emphasis over the input space possibly leading to different conclusions. 

\textbf{Future Work:} 
We are excited about the future work opened up by the here-proposed gradient estimates. 
First, we suspect that part of the reason for the slow convergence rate of the gradient estimator is that we ``start from scratch" at every depth of the tree, overwriting the previous estimator. 
Perhaps an improved convergence rate could be obtained by combining the gradient estimate from multiple depths rather than overwriting it.
Next, while we have conducted some preliminary analyses in Appendix \ref{app:class}, more work is needed to understand the properties of these estimates in classification problems with categorical inputs and missing data. 
We are also interested in the possibility of estimating higher order derivatives from tree structure. 
Furthermore, extending the class of functions approximated to nondifferentiable functions is of interest: does the quantity proposed here converge in such a case, and to what?
But we are most excited about the potential to bring in other gradient-based techniques to tree-based methods.
For instance, physics-informed machine learning (e.g. \citet{raissi2019deep,raissi2017machine}) has been making waves in the fluid dynamics community  \citep{cai2021physics} among other applications, where they allow the analyst to use information about function derivatives to improve predictions of differentiable models such as neural networks or Gaussian processes. 
Our gradient estimator opens up the possibility of deploying this technique in the context of regression trees.

\begin{ack}
    
I gratefully acknowledge funding from the NSF and NGA via DMS\#2428033 as well as the NSF and AFOSR via DMS\#2529277.
The computational resources for this work were provided by the Unity Research Computing Platform, a multi-institutional cluster lead by University of Massachusetts Amherst, the University of Rhode Island, and University of Massachusetts Dartmouth. 
This project was inspired by discussions at the 2024 AISTATS conference.

\end{ack}

%
%

%

\bibliography{example_paper}
\bibliographystyle{apalike}

\newpage
\appendix
\onecolumn






\newpage

\section{Proofs}
\label{app:proofs}

We begin by stating and discussing our assumptions.

A major difficulty in establishing theoretical results for tree-based methods is establishing the behavior of the thresholds as a function of the data generating process.
For our purposes, the critical assumption is that the volume of the leaves shrink at an exponential rate in the number of splits made. 
We will study a scheme which achieves this by cycling between variables to split along at each depth, and subsequently splitting at the median of datapoints in each node along the split variable, such that about half of the data in the current node lie in each child node.
To illustrate, the first split in the root node is made on variable 1, then at depth two, both child nodes have about half the data and themselves split along variable 2, and subsequently the four child nodes, each with a fourth of the data, split along variable 3. 
When we reach a depth of $P+1$, the $2^{P+1}$ splits will again be along variable 1, and each node has about $\frac{N}{2^{P+1}}$ many points in it. 

This is of course typically not be how a regression tree is fit in the supervised setting, which will typically calculate splits taking $y$ values into account.
Establishing the convergence rate of the TBGE in such settings is left as future work. 
We believe that even the present proof, provided in the simplified setting, nevertheless provides valuable intuition as to the advantages and limitations of the TBGE.

Our first assumption concerns the splitting rate.
\begin{ass0}
\label{ass:sn}
Let $S(N)$ denote the depth as a function of sample size.
\begin{enumerate}
    \item $\underset{N\to\infty}{\lim} S(N) = \infty$. 
    \item $\underset{N\to\infty}{\lim} \frac{S(N)}{\log N} = 0$.
\end{enumerate}
\end{ass0}

This very slow rate of splitting is necessary to ensure that the number of datapoints in each node can grow to infinity while still eventually achieving arbitrarily small nodes.

Our next assumption is on the process which generates the feature variables.

\begin{ass1}
\label{ass:splits}
    The input points $\x_n$ are sampled from a distribution such that for any $\epsilon > 0$, there are constants $C_-,C_+$ giving:
    \begin{align*}
        P\left(C_- 2^{-\frac{D_i}{P}}\leq |u^i_p - l^i_p| \leq C_+ 2^{-\frac{D_i}{P}}\right)>1-\epsilon \quad \forall p\in\{1,\ldots,P\}\,,
    \end{align*}
    where $D_i := 1+\lfloor \log_2(i+1) \rfloor$ is the depth of node $i$.
\end{ass1}

This assumption tells us that the width of the leaf nodes will shrink exponentially. 
This assumption is most likely to be violated in observational contexts, where the feature variables may live on some manifold within $[0,1]^P$ and thus leave gaps within the unit cube unfilled. 
Naturally, the gradient estimates for regions without data will not behave well. 
However, this assumption may be satisfied in contexts where the feature variables may be set by the experimenter. 
In particular, this property holds for the uniform distribution, as we now discuss.

\begin{rem1}
Assumption 1 is satisfied if $\x_1, \ldots, \x_N \overset{i.i.d}{\sim} U([0,1]^P)$ and $S(N)$ satisfies Assumption 0.
\end{rem1}
\begin{proof}
    Start with the case that $P=1$.
    Here, we have simply uniform data on $[0,1]$, and are making $S(N)$ many recursive splits so as to keep an equal number of points in each bin.
    The distribution of the bin borders is given by the order statistic distributions. 
    For uniform data, it is known that the difference between the $j^{\textrm{th}}$ and $k^{\textrm{th}}$ order statistics is given by a Beta distribution with $\alpha=k-j$ and $\beta=N+1-(k-j)$ when $k>j$. 
    As each bin has $\frac{N}{2^{S(N)}}$ observations, $k-j=\frac{N}{2^{S(N)}}$ and the marginal distribution of a bin width is thus:
    \begin{align*}
        u^i - l^i \sim \textrm{Beta}\left(\frac{N}{2^{S(N)}}, N+1-\frac{N}{2^{S(N)}}\right) \,.
    \end{align*}
    We thus have that:
    \begin{enumerate}
        \item $\mathbb{E}[u^i - l^i] = 
        \left(\frac{N}{N+1}\right) 
        \frac1{2^{S(N)}}
        \in
        \Theta\left(\frac1{2^{S(N)}}\right)
        $ \,,
        \item $\mathbb{V}[u^i - l^i] = 
        \left(\frac{N^2}{(N+1)^2(N+2)}\right)
        \left(
        1+\frac1{N}-\frac1{2^{S(N)}}
        \right)
        \frac1{2^{S(N)}}
        \in
        \Theta\left(\frac1{N 2^{S(N)}}\right)
        $ \,,
    \end{enumerate}
    where as usual $\Theta$ denotes the class of functions bounded on both sides by constant multiples of its argument.
    
    We next develop a Chebyshev bound, valid for all $k>1$:
    \begin{align*}
        & P\left(
        \left|
        (u^i-l^i) - \frac{1}{2^{S(N)}}
        \right|
        \leq 
        \frac{k}{\sqrt{N2^{S(N)}}}
        \right)
        \geq 1-\frac1{k^2} 
        \\ & \iff
        P\left(
        \left|
        2^{S(N)}(u^i-l^i) - 1
        \right|
        \leq 
        k \sqrt{\frac{2^{S(N)}}{N}} 
        \right)
        \geq 1-\frac1{k^2} \,.
    \end{align*}

    We now fix some $\epsilon>0$, and set $k=\frac1{\sqrt{\epsilon}}$, such that:
    \begin{align*}
        P\left(
        \left|
        2^{S(N)}(u^i-l^i) - 1
        \right|
        \leq 
        \sqrt{\frac{2^{S(N)}}{N}} 
        \frac1{\sqrt{\epsilon}}
        \right)
        \geq 1-\epsilon
        \,.
    \end{align*}
    
    Under our assumption that $S(N) \in o(\log N)$, we have that $\frac1{\sqrt{N 2^{S(N)}}} \in o(\frac1{2^{S(N)}})$, such that $\sqrt{\frac{2^{S(N)}}{N}}$ may be made arbitrarily small for sufficiently large $N$.
    Fixing some $0<\delta<1$,
    choose $N$ such that 
    $
    \sqrt{\frac{2^{S(N)}}{N}} < \delta \sqrt{\epsilon}
    $
    .
    Then:
    \begin{align*}
        P\left(
        \left|
        2^{S(N)}(u^i-l^i) - 1
        \right|
        \leq 
        \delta
        \right)
        \geq 1-\epsilon
        \,,
    \end{align*}
    which gives us that $2^{-S(N)} (1-\delta) \leq u^i-l^i \leq 2^{-S(N)}(1+\delta)$ with high probability.

    When $P>1$, we alternate between splitting along each variable. 
    The marginal distribution of $x_p$ remains uniform within each split node, which allows us to reduce the general dimensional case to the one dimensional case.
    However, each variable is split only every $P^\textrm{th}$ pass, which rescales the $S(N)$ term in the exponent by $\frac1{P}$.
\end{proof}

Our final assumptions are on the function and error process.

\begin{ass2}
\label{ass:tbge}
Let $y_n = f(\x_n)+\epsilon_n$. We assume that:
\begin{enumerate}
    \item The estimate of the mean of $f(\x)$ over some interval $[\mathbf{a},\mathbf{b}]$ converges at the usual square root rate; this is achieved for example if the error terms $\epsilon_1,\ldots,\epsilon_N$ are such that $\mathbb{E}[\epsilon_n] = 0$; $\mathbb{V}[\epsilon_n] \leq \infty$ and $\epsilon_{n_1}\indep \epsilon_{n_2} \forall n_1,n_2\in\{1,\ldots,N\}$.
    \item $f$ is twice differentiable with $\Vert\nabla^2 f (\x)\Vert \leq H \,\, \forall \x\in[0,1]^P$.
\end{enumerate}
\end{ass2}

The first assumption gives us that the sample means in a given node converge to the integral of the function over that node normalized by the node's volume. 
The second assumption allows us to bound the error of a 1st order Taylor approximation. 



We are now prepared to prove our results.

\begin{thm41}\label{thm:grad_conv}
Let $S(N)$ denote the number of splits in the tree as a function of the sample size.
Let 
$
\frac{\tilde{\partial}^N f(\x)}{\tilde{\partial} x_{p}} 
$
denote the TBGE at point $\x\in[0,1]^P$ with sample size $N$.
Under Assumptions 0, 1, and 2:
\begin{equation*}
    \frac{\tilde{\partial}^N f(\x)}{\tilde{\partial} x_{p}} 
    = 
    \frac{\partial f(\x)}{\partial x_{p}}
    +
    O_P\left(
    \frac{2^{\frac{P+2}{2P}S(N)}}{\sqrt{N}}
    +
    P2^{-\frac{S(N)}{P}}
    \right)
    \,.
\end{equation*}
\end{thm41}
\begin{proof}

At each node $i$, the quantity $\gamma_i$ serves as our estimate of the partial derivative in the $\sigma_i$ direction near that node.
For a fixed $i$, we do not expect convergence of $\gamma_i$ to any gradient quantity.  
However, we do expect convergence as we go down the tree depth for suitably chosen $S(N)$.

We wish to estimate the $p$ partial derivative at some location $\x\in[0,1]^P$ given a sample of size $N$.
Given the depth $S(N)>P$, the most recent split along variable $p$ is given by the deepest node $i$ of depth $K$ no greater than $S(N)$ which splits is along variable $p$ (that is, $i \equiv p \pmod{P}$) and which contains $\x$ (that is, $B^K(\x) = i$).
Our task is to show that this sequence of $\gamma_i$, with $i$ thusly viewed implicitly as a function of $N$, will converge to $\frac{\partial f(\x)}{\partial x_p}$.

Recall the definition of the finite-difference like estimate:
\begin{align}
    \gamma_i
    =
    \frac{2(\hat\mu_r^i - \hat\mu^i_l)}{u^i_{\sigma_i} - l^i_{\sigma_i}}
    \,,
\end{align}
and $\sigma_i = p$ by our choice of $i$.

$\hat\mu_l^i$ is the mean of the data falling within the left child of node $i$. 
By our assumption, that within any node, the sample mean converges to the population mean at the usual rate, we have that:
\begin{align}
    \hat\mu_l^i = 
    \mu_{l}^i + O_P\left(\sqrt{\frac{2^{S(N)}}{N}}\right) 
    :=
    \underset{\x\sim U\left([\lb^{c_l^i}, \ub^{c_l^i}]\right)}{\mathbb{E}}
    \left[
    f(\x)
    \right]
    + O_P\left(\sqrt{\frac{2^{S(N)}}{N}}\right) 
    \,,
\end{align}
where the second equation serves to give our definition of $\mu_i^l$, a quantity independent of our dataset and instead dependent on $f$, the mean function of the relationship being approximated.
Here, the expectation is taken with respect to the uniform distribution over the bounds of the left child node of $i$, which as before we denote by $c^i_l$.

Next we approximate $f$ at an arbitrary point $z\in[\lb^i,\ub^i]$ by expanding it about $\x$ (or any other point in $[\lb^i,\ub^i]$), using the Lagrange form of the remainder with $\xi$ dependent on $\z$:
\begin{align}
    f(\z) = 
    f(\x) + \nabla f(\x)^\top 
    \left(
    \z - \x
    \right)
    +
    \frac12
    (\z-\x)^\top
    \nabla^2 f(\xi)
    (\z-\x)
    \,,
\end{align}
and plug this into the definition of $\mu_l^i$:
\begin{align*}
    & \mu_l^i
    =
    \underset{\z\sim U\left([\lb^{c_l^i}, \ub^{c_l^i}]\right)}{\mathbb{E}}
    [
    f(\x) + \nabla f(\x)^\top 
    \left(
    \z-\x
    \right)
    +
    \frac12
    (\z-\x)^\top
    \nabla^2 f(\xi)
    (\z-\x)
    ]
    \\ & = 
    f(\ub^i) + 
    \frac12
    \nabla f(\x)^\top 
    \left(
    \lb^{c_l^i}+\ub^{c_l^i}
    \right)
    - 
    \nabla f(\x)^\top 
    \x
    +
    \frac12
    \underset{\z\sim U\left([\lb^{c_l^i}, \ub^{c_l^i}]\right)}{\mathbb{E}}
    [
    (\z-\x)^\top
    \nabla^2 f(\xi)
    (\z-\x) 
    ]\,.
\end{align*}
We needed only the linearity of expectation to evaluate the first two terms. 
Since we likewise have:
\begin{align*}
    & \mu_r^i
    =
    f(\x) + 
    \frac12
    \nabla f(\x)^\top 
    \left(
    \lb^{c_r^i}+\ub^{c_r^i}
    \right)
    - 
    \nabla f(\x)^\top 
    \ub^i
    +
    \frac12
    \underset{\z\sim U\left([\lb^{c_r^i}, \ub^{c_r^i}]\right)}{\mathbb{E}}
    [
    (\z-\x)^\top
    \nabla^2 f(\xi)
    (\z-\x) 
    ]\,.
\end{align*}
Taking the difference yields:
\begin{align*}
    & \mu_r^i - \mu_l^i = 
    \nabla f(\x)^\top 
    \left(
    \lb^{c_r^i}
    -\lb^{c_l^i}
    +
    \ub^{c_r^i}
    -\ub^{c_l^i}
    \right)
    + R \,,
\end{align*}
where 
\begin{align*}
    R = 
    \frac12
    \underset{\z\sim U\left([\lb^{c_r^i}, \ub^{c_r^i}]\right)}{\mathbb{E}}
    [
    (\z-\x)^\top
    \nabla^2 f(\xi(\x))
    (\z-\x) 
    ]\,
    -
    \frac12
    \underset{\z\sim U\left([\lb^{c_l^i}, \ub^{c_l^i}]\right)}{\mathbb{E}}
    [
    (\z-\x)^\top
    \nabla^2 f(\xi(\x))
    (\z-\x) 
    ]\,.
\end{align*}
As siblings, the bounds of $c^i_r$ and $c^i_l$ agree except for along the $p$ dimension, such that:
\begin{align*}
    & \nabla f(\x)^\top\left(
    \lb^{c_r^i}
    -\lb^{c_l^i}
    +
    \ub^{c_r^i}
    -\ub^{c_l^i}
    \right)
    \\ &=
    \sum_{k\neq p}
    \frac{\partial f(\x)}{\partial x_k}
    (
    l^i_k-l^i_k
    + u^i_k-u^i_k
    )
    +
    \frac{\partial f(\x)}{\partial x_p} 
    (\tau^i - l_p^i + u_p^i - \tau^i)
    = \frac{\partial f(\x)}{\partial x_p} (u_p^i - l_p^i) \,.
\end{align*}
Therefore:
\begin{align*}
    \mu_r^i - \mu_l^i 
    = 
    \frac12
    \frac{\partial f(\x)}{\partial x_p} (u_p^i - l_p^i)
    + R \,,
\end{align*}
such that
\begin{align*}
        \frac{2(\mu_r^i - \mu_l^i)}{u_p^i-l_p^i}
        =
     \frac{\partial f(\x)}{\partial x_p}   
     +
     \frac{R}{u_p^i-l_p^i} \,.
\end{align*}

Let's now work to develop $R$.
Note that:
\begin{align}
    & \left|
    \underset{\z\sim U\left([\lb^{c_r^i}, \ub^{c_r^i}]\right)}{\mathbb{E}}
    [
    (\z-\x)^\top
    \nabla^2 f(\xi(\z))
    (\z-\x) 
    ]\,
    \right|
    \leq
    \underset{\z\in[\lb^{c_r^i}, \ub^{c_r^i}]}{\max}
    \left|
    (\z-\x)^\top
    \nabla^2 f(\xi(\z))
    (\z-\x) 
    \right|
    \\ & \overset{C.S.}{\leq}
    \underset{\z\in[\lb^{c_r^i}, \ub^{c_r^i}]}{\max}
    \left\Vert
    \z-\x
    \right\Vert
    \left\Vert
    \nabla^2 f(\xi(\z))
    (\z-\x) 
    \right\Vert
    \overset{\textrm{def } \Vert\cdot\Vert}{\leq}
    \underset{\z\in[\lb^{c_r^i}, \ub^{c_r^i}]}{\max}
    \left\Vert
    \z-\x
    \right\Vert^2
    \left\Vert
    \nabla^2 f(\xi(\z))
    \right\Vert \,.
\end{align}
The first inequality comes from the absolute integral being less than the max of the absolute integrand, the second comes from Cauchy-Schwartz, and the third from the definition of the operator norm.
Bounding the max of the product of positive quantities by the product of the maxes and invoking our bound on 
$
\left\Vert
\nabla^2 f(\xi(\z))
\right\Vert
$
will be our next steps, yielding:
\begin{align*}
    \underset{\z\in[\lb^{c^r_i}, \ub^{c^r_i}]}{\max}
    \left\Vert
    \z-\x
    \right\Vert^2
    \left\Vert
    \nabla^2 f(\xi(\z))
    \right\Vert 
    \leq 
    H
    \left\Vert
    \ub^i - \lb^i
    \right\Vert^2 \,.
\end{align*}
Making similar moves for the left child node gives us that:
\begin{align*}
    \frac{|R|}{u_p^i-l_p^i} \leq H \frac{\Vert \ub^i-\lb^i\Vert^2}{u_p^i-l_p^i}
    \,.
\end{align*}
Under our assumptions on the splitting process, we have that $C_- 2^{-\frac{S(N)}{P}} \leq |u_k^i - l_k^i| \leq C_+ 2^{-\frac{S(N)}{P}}$ for some $C_-,C_+>0$ with high probability.
Thus, also with high probability:
\begin{align*}
    \frac{\Vert \ub^i - \lb^i\Vert^2 }{u_p^i - l_p^i} 
    \leq
    \frac
    {PC_+^2 2^{-2 \frac{S(N)}{P}}}
    {C_- 2^{-\frac{S(N)}{P}}}
    =
    \left(P\frac{C_+^2}{C_-}\right)
    2^{-\frac{S(N)}{P}} \,,
\end{align*}
and therefore
\begin{align*}
        \frac{2(\mu_r^i - \mu_l^i)}{u_p^i-l_p^i}
        =
     \frac{\partial f(\x)}{\partial x_p}   
     +
     O_P\left(
     P
     2^{-\frac{S(N)}{P}}
     \right) \,.
\end{align*}

We thus have a bound on the difference between the true functional mean and the gradient, but have thus far been ignoring error in the estimation of these means.
To bring this into the equation, we begin by noting that:
\begin{align*}
    \hat \mu_r^i = \mu_r^i 
    + O_P\left(\sqrt{\frac{2^{S(N)}}{N}}\right) 
    \implies 
    \frac{\hat \mu_r^i}{u_p^i - l_p^i} = \frac{\mu_r^i}{u^i_p-l^i_p}  
    + O_P\left(\sqrt{\frac{2^{\frac{P+2}{P}S(N)}}{N}}\right) 
\end{align*}
and similarly for $\hat\mu_l^i$, which gives us that:
\begin{align*}
    \frac{2(\hat\mu_r^i - \hat\mu_l^i)}{u_p^i - l_p^i}
    =
    \frac{2(\mu_r^i - \mu_l^i)}{u_p^i - l_p^i}
    + O_P\left(\frac{2^{\frac{P+2}{2P}S(N)}}{\sqrt{N}}\right)  \,,
\end{align*}
Combining with our earlier work yields:
\begin{align*}
    \frac{2(\hat\mu_r^i - \hat\mu_l^i)}{u_p^i - l_p^i}
    =
     \frac{\partial f(\x)}{\partial x_p}   
     +
     O_P\left(
    \frac{2^{\frac{P+2}{2P}S(N)}}{\sqrt{N}}
    +
    P2^{-\frac{S(N)}{P}}
    \right) \,.
\end{align*}

\end{proof}

%

\begin{thm42}
Let $h:\mathbb{R}^P\to\mathbb{R}^P$ be bounded in the sense that $h(\z) \leq C_h \Vert\z\Vert^{a_h}$ for $a_h>0$.
Under Assumptions 0, 1 and 2, the Partition-Based estimator converges to the true integro-differential quantity as the observation sample size diverges. That is,
\begin{align*}
    \lim_{N\to\infty} \hat{\mathcal{I}}_{PB}(f) = 
    \lim_{N\to\infty} \sum_{i\in\N_{K-1}} 
    h(\G^i)\mu([\lb^i,\ub^i])
    \,.
    =
    \int h(\nabla f(\x)) d\mu(\x) 
    =
    \mathcal{I}(f) \,.
\end{align*}
\end{thm42}

\begin{proof}
    Since the function sequence $\gv^k(\x) := G_{B^k(\x)}$ converges pointwise to $\nabla f(\x)$ by Proposition 1, we need only establish a function $H(\x)$ which dominates $\gv^k(\x)$ and apply the Dominated Convergence Theorem.
    
    To this end, denote by $C^r,C^l$ the child nodes of node $i$ and examine its gradient estimator's $p$th entry, given by:
    \begin{equation}
    	\frac{2(\mu_i^r-\mu_i^l)}{u_p^i-l_p^i}
    	\overset{N\to\infty}{\to}
    	    	\frac{2\big(
    	    		\frac{1}{|\N_{C^r}|}\int_{\N_{C^r}} f(\x)d\x
    	    		-
    	    		\frac{1}{|\N_{C^r_i}|}\int_{\N_{C^l_i}} f(\x)d\x
    	    		\big)}{u_p^i-l_p^i} \, .
    \end{equation}

	The magnitude of this difference in averages is bounded by the magnitude of the difference of extremes:
	
	\begin{equation}
    	 |\frac{1}{|\N_{C^r_i}|}\int_{\N_{C^r_i}} f(\x)d\x
		-
		\frac{1}{|\N_{C^r_i}|}\int_{\N_{C^l_i}} f(\x)d\x|
		\leq
		\underset{(\x_1,\x_2)\in \N_{C^r_i}\times\N_{C^l_i}}{\max} |f(\x_1)-f(\x_2)| \,.
	\end{equation}

	But since $f$ is continuously differentiable, it is also Lipschitz continuous (call the constant $L$), and we have that:
	
	\begin{equation}
		\underset{(\x_1,\x_2)\in \N_{C^r_i}\times\N_{C^l_i}}{\max} |f(\x_1)-f(\x_2)|
		\leq
		 L \Vert \x_1 - \x_2 \Vert_2 \leq P L \Vert \x_1 - \x_2 \Vert_{\infty} \, .
	\end{equation}

	Therefore:
	\begin{equation}
		    	\left|\frac{2(v_{C^r_i}-v_{C^l_i})}{u_p^i-l_p^i}\right| 
		    	\leq
		    	\left|\frac{2(P L \Vert \x_1 - \x_2 \Vert_{\infty} )}{u_p^i-l_p^i}\right| 
		    	\leq 
		    	2PL \, .
	\end{equation}

	Hence $\Vert g^k(\x)\Vert_2 \leq 2P^2L$, 
    , and thence $C_h(2P^2L)^{a_h}$ bounds $h(g^k(\x))$. 
	Since the integral is over the unit hypercube, the constant function is integrable and we can apply the Dominated Convergence Theorem to yield the desired result.
    
\end{proof}

\begin{theorem}
\label{thm:mce}
Under Assumptions 1 and 2, the Monte Carlo estimator converges to the true integro-differential quantity as the Monte Carlo sample size and the observation sample size diverge. That is,
\begin{align*}
    \lim_{N,M\to\infty} \hat{\mathcal{I}}_{MC}(f) 
    =
    \lim_{N,M\to\infty}
    \frac{1}{\mcs}
    \sum_{\x_m\sim\mu}
    h(\tilde\nabla f(\x)) 
    = 
    \int h(\nabla f(\x)) d\mu(\x) 
    =
    \mathcal{I}(f) \,.
\end{align*}
\end{theorem}
\begin{proof}
    This follows from the fact that 
    \begin{align}
        & \lim_{N\to\infty}
        \lim_{M\to\infty}
        \frac{1}{M}
        \underset{\x_m\sim\mu}{\sum}
        h(\tilde{\nabla} f(\x_m))
        =        
        \\ & =
        \lim_{N\to\infty}
        \int_{\x\in[0,1]^P}
        h(\tilde{\nabla} f(\x))
        d\mu
        =
        \lim_{N\to\infty}
        \sum_{i\in\D_k} h(\G_i) \mu([\lb^i,\ub^i])
    \end{align}
    and an application of Theorem 4.2.
\end{proof}

%
%

\subsection{Implications for Active Subspaces and Integrated Gradients}
\label{sec:implications}

We conclude this section by explicitly stating the implications of these results for Tree-based Active Subspace (TBAS) estimation.
\begin{corollary}
    Let $K(N)$ denote the depth of the regression tree as a function of $N$.
    Under Assumptions 1 and 2, we have that
    \begin{align*}
        & \lim_{N\to\infty} 
        \sum_{i\in\D_{K(N)-1}} \G_{i}\G_{i}^\top \mu(\N_i) 
        =
        \int_{[0,1]^P} \nabla f (\x) \nabla f(\x) d\mu(\x) \, .
    \end{align*}
\end{corollary}
\begin{proof}
    Follows from Theorem \ref{thm:pbe}.
\end{proof}
As well as for Tree-Based Integrated Gradient (TBIG) estimation.
\begin{corollary}
    Let $K(N)$ denote the depth of the regression tree as a function of $N$.
    Under Assumptions 1 and 2, we have that, assuming that $u_n$ are iid uniform on [0,1]:
    \begin{align*}
        \lim_{k,N,M\to\infty}  
        (\x-\x^*)
        \odot 
        \frac{1}{M}
        \sum_{m=1}^M 
        \tilde\nabla f\left(u_m \x + (1-u_m) \x^*\right)
        \\= 
        (\x-\x^*) \odot \int_{\alpha=0}^1 \nabla f(\alpha\x + (1-\alpha)\x^*) d\x \,.
    \end{align*}
\end{corollary}
\begin{proof}
    Follows from Theorem \ref{thm:mce}.
\end{proof}

\section{Details of Numerical Experiments}

This section gives additional details and discussion of the numerical results presented in Section \ref{sec:num}. All tree-based models are estimated using Scikit-Learn \cite{scikit-learn}. 

\subsection{Rotation Prediction Study Additional Details}\label{app:rot}

\begin{figure}
	\centering
	\textbf{Tree Depth 4}
	\includegraphics[width=0.95\textwidth]{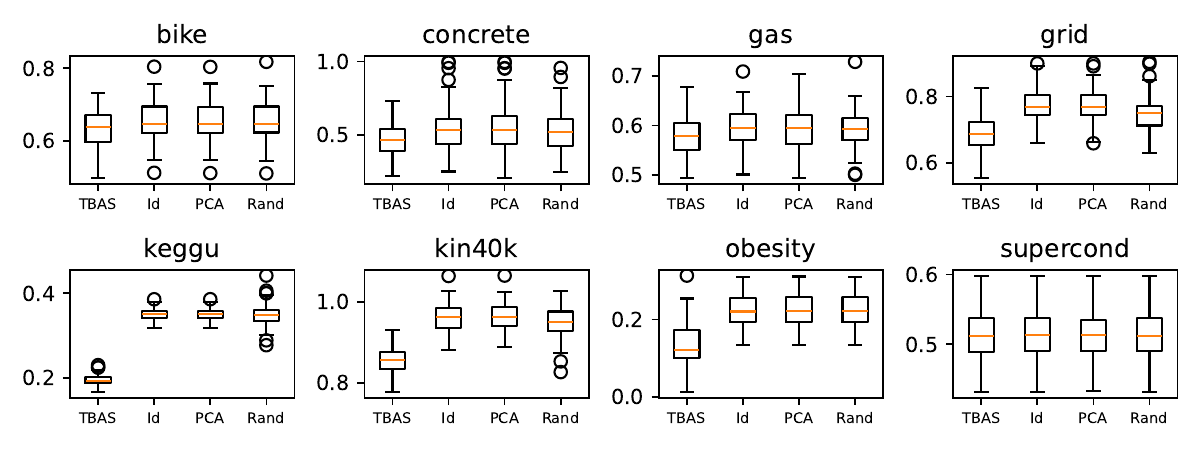}
	
	\textbf{Tree Depth 8}
	\includegraphics[width=0.95\textwidth]{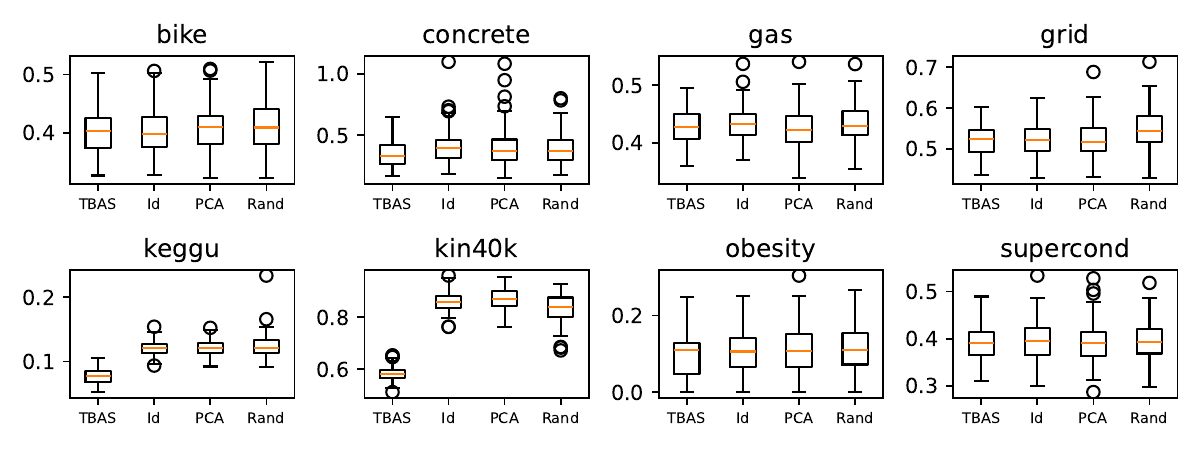}
	
	\textbf{Random Forest Depth 4}
	\includegraphics[width=0.95\textwidth]{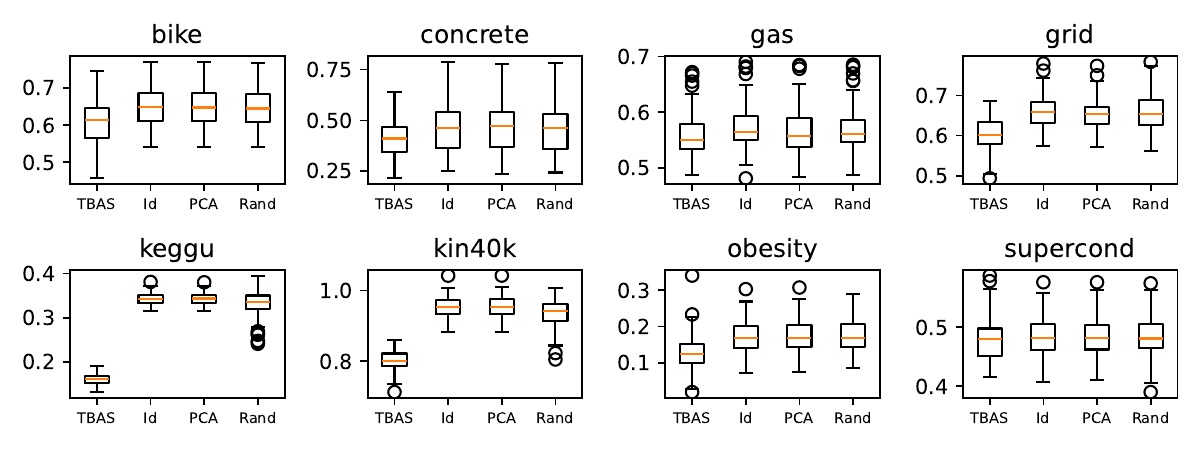}
	\caption{Boxplots corresponding to Table \ref{tab:rot}.}
	\label{fig:pred_bakeoff2}
\end{figure}

In this study, when performing a rotation, we used only the $\sqrt{P}$ many PCA components or Active Subspace dimensions, and appended these to the original design matrix as new variables. 
We measure prediction error using 100-fold cross validation.

The table below gives the parameters of the datasets used for the prediction study of Section \ref{sec:num_pred}.

{
    \begin{center}
	\small
	\begin{tabular}{|lrrr|}
		\hline
		Name & N & P & URL \\
		\hline
		concrete & 1,030 & 9 & \url{https://archive.ics.uci.edu/dataset/165/} \\
		kin40k & 40,000 & 9 & \url{https://github.com/alshedivat/keras-gp/kgp/datasets/kin40k.py} \\
		keggu & 65,554 & 28 & \url{https://www.genome.jp/kegg/pathway.html} \\
		bike & 17,379 & 13 & \url{https://archive.ics.uci.edu/dataset/560/} \\
		obesity & 2,111 & 24 & \url{https://archive.ics.uci.edu/dataset/544/} \\
		gas & 36,733 & 12 & \url{https://archive.ics.uci.edu/dataset/224} \\
		grid & 10,000 & 13 & \url{https://archive.ics.uci.edu/dataset/471/} \\
		supercond & 21,263 & 82 & \url{https://archive.ics.uci.edu/dataset/464/} \\
		\hline
	\end{tabular}	
    \end{center}
}

We also provide boxplots of Cross Validation errors in Figure \ref{fig:pred_bakeoff2}.
Running this study took about five hours on a 40 core Ubuntu machine with 128 GB of RAM. 

\subsection{Active Subspace Estimation Study Additional Details}\label{app:subspace}

In Section \ref{sec:num_smol_active}, We randomly sampled a unit vector $\mathbf{a}$ from the uniform distribution over directions and then sampled input points uniformly at random on the unit cube.
We subsequently evaluated the function $f(\x) = \cos(6\pi(\mathbf{a}^\top(\x-0.5))$ which was treated as the noiseless observed response $\y$.
We compared the estimates with using the angle each made with $\mathbf{a}$.
This experiment was repeated 20 times.

We used the implementation of PRA provided with the pypi package \texttt{PSDR}\footnote{https://psdr.readthedocs.io/en/latest/}.
For GP-based active subspace estimation, we used the CRAN package \texttt{activegp}.
We used all default settings for the PRA method as well as for the GP method.
For the DASM, we used a neural network with an additional layer of width 512 subsequent to the active subspace layer, and used gradient descent with a step size of $10^{-3}$ on the Mean Squared Error cost function. 
This neural network was implemented in JAX \cite{jax2018github}. 
In addition to the quantitative advantages enjoyed by the tree-based method of active subspace estimation, we would also like to note that like the GP-based method, and unlike the PRA and DASM, it provides an estimate of the entire active subspace matrix, rather than simply a basis for the active subspace.
This is important for two reasons.
First, with the active subspace matrix in hand, we can create analogs of PCA scree plots to determine what dimension of the active subspace is most desirable, or to get some idea of how much information is being lost in, say, a two dimensional visualization.
And secondly, it allows us to decide on an active subspace dimension \textit{after} having seen the data rather than before, without requiring the estimation procedure to be re-run.

Running this study took about eight hours on a 40 core Ubuntu machine with 128 GB of RAM.

\subsection{NHEFS Data Analysis Details}\label{app:mortality}

This table presents the first 3 eigenvectors of the mortality data analysis, restricting to the top 9 variables with highest coefficients.
The first eigenvector captures almost entirely the age variable, which \cite{lundberg2019explainable} also found to be most important. 
We see that the second eigenvector is evenly distributed across sex and one of the urine variables.
The third eigenvector is dominated by the urineDark variable.

\begin{center}
\begin{tabular}{|l|r|r|r|r|r|r|r|r|r|}
	\hline
	Eigenvector & age & urineDark & sex & urineNeg & SGOT & hemoglobin & urineAlb & total & physical \\
    \hline
	1 & -1.00 & -0.00 & 0.01 & 0.01 & -0.00 & -0.00 & 0.00 & -0.00 & -0.00  \\
	2 & 0.01 & -0.15 & 0.56 & 0.53 & -0.28 & -0.30 & 0.30 & -0.16 & 0.09  \\
	3 & -0.00 & -0.96 & -0.03 & -0.11 & 0.16 & 0.07 & 0.02 & 0.04 & 0.07  \\
	\hline
\end{tabular}
\end{center}

When producing the right panel of Figure \ref{fig:eig}, we used the prediction after accounting for the effect of age in order to demonstrate the change in the predictive surface over the second and third eigenvalues. 

\section{Additional Numerical Experiments}
\subsection{Empirical Gradient Estimate Quality}

We present the results of a simulation showing the empirical performance of TBGE in estimating gradients under various tree depth, sample sizes, dimensions and gradient densities on the function $f(x) = \log(1+\mathbf{a}^\top\mathbf{x})$ with nonzero elements of $\mathbf{a}$ generated from an iid Gaussian.
Figure \ref{fig:conv} contains heatmaps showing the results for all combinations of a $25\%$ and $100\%$ dense gradient and a depth 4 and 12 tree.
In each figure, the x-axis gives the sample size and the y-axis the problem dimension.
The color represents the angle between the true and estimated gradient.
We see that a deep tree is able to decrease error as sample size increases while a shallow one cannot, and also that error decreases much faster for the sparse gradient.

\begin{figure}
    \centering
    \includegraphics[width=0.8\linewidth]{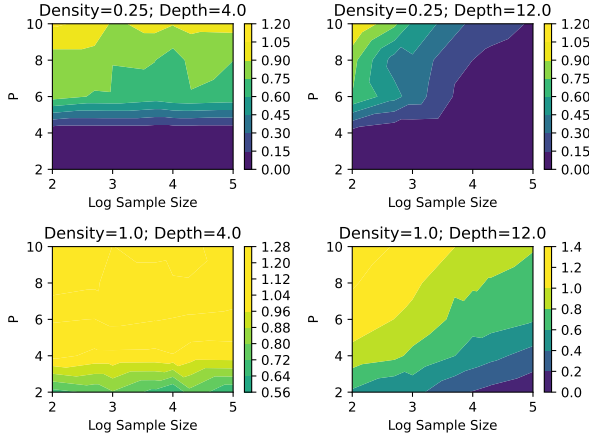}
    \caption{Convergence of Gradient Estimates in Sample Size.}
    \label{fig:conv}
\end{figure}

\subsection{Gradient Error Comparison against GP Surrogate.}

In this section we conduct a study similar to Section \ref{sec:num_smol_active}, but focus directly on the gradient quality rather than the active subspace quality.
That is, we compare the accuracy of gradient estimates produced by a regression tree against those provided by a Gaussian process for a fixed \textit{compute time}, rather than a fixed dataset size.
The idea is that on simpler simulations, we may be able to produce an abundance of data and that trees might be useful in this scenario.
In Figure \ref{fig:gp_compare}, we find that indeed the tree can produce useful estimates extremely quickly for larger sample sizes, leading to some comparative advantage compared to a GP fit on a smaller sample size (but taking even longer to make predictions).

\begin{figure}
    \centering
    \includegraphics[width=0.8\linewidth]{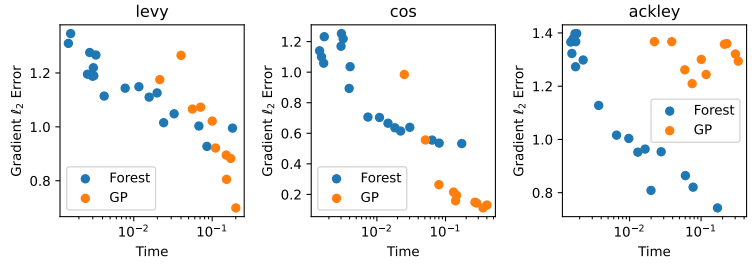}
    \caption{Gradient Error Comparison against GP surrogate.}
    \label{fig:gp_compare}
\end{figure}

\subsection{Active Subspace Estimation Performance under Noise}

This section repeats the experiments of Section \ref{sec:num_smol_active} where the output data are contaminated under small i.i.d. Gaussian noise with standard deviation $0.1$.
The results, presented in Figure \ref{fig:noisy} are very similar to the noiseless case but are nevertheless included for completeness.

\begin{figure}
    \centering
    \includegraphics[width=0.8\linewidth]{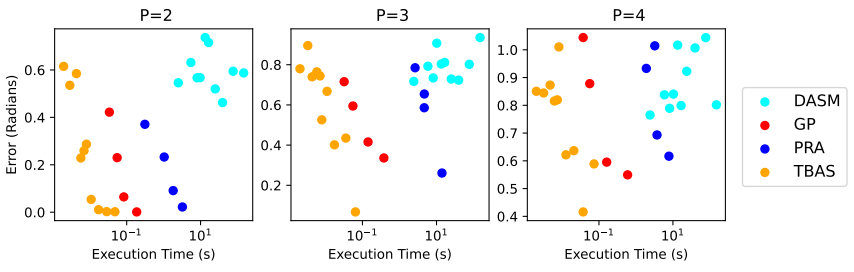}
    \caption{Active Subspace Estimation with Gaussian Noise (std=0.1;iid).}
    \label{fig:noisy}
\end{figure}

\subsection{Gradient Estimate under Input Variable Correlation}

Our numerical experiments on real data suggest our proposed method can operate under correlated inputs, and in this section we conduct an experiment to more closely investigate quantitatively the impact of input-variable correlation on gradient estimate quality.
Sampling data from a truncated normal in 5D with correlation varying from 0 to 0.99, we evaluated estimates of the Ackley function's gradient at 100 random points.
The results are presented in Figure \ref{fig:corr}.
We see that correlation can have some deleterious effects on the gradient estimates; however, it requires a very high level of correlation before the estimates are significantly affected.

\begin{figure}
    \centering
    \includegraphics[width=0.5\linewidth]{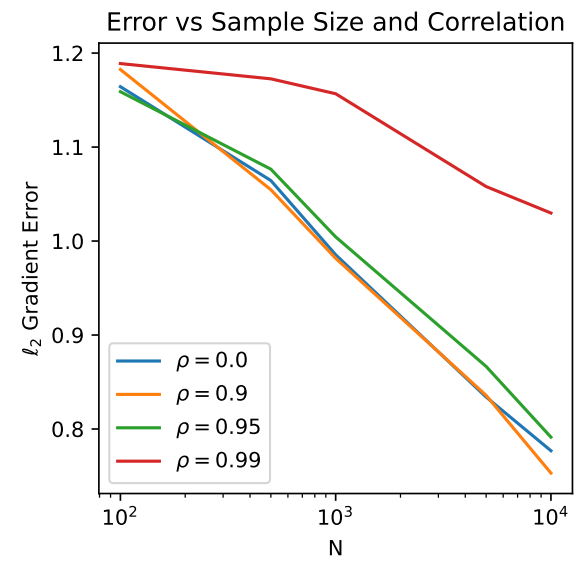}
    \caption{The Effect of Correlation on Gradient Estimates.}
    \label{fig:corr}
\end{figure}

\subsection{Numerical Experiments on Classification Trees}\label{app:class}

We repeat the experiments of Section \ref{sec:num_pred}, but now by replacing each regression problem with a classification problem by assessing whether a given observation falls above or below the median observation.
The results are given in Figure \ref{fig:class_bakeoff}.
Intriguingly, the results are significantly less promising for the TBAS method, despite the fact that by construction, there is structure in the data that TBAS could possibly exploit. 
This indicates that there may be special considerations to be taken in deploying this methodology to classification problems.

\begin{figure}
	\centering
	\textbf{Tree Depth 4}
	\includegraphics[width=0.95\textwidth]{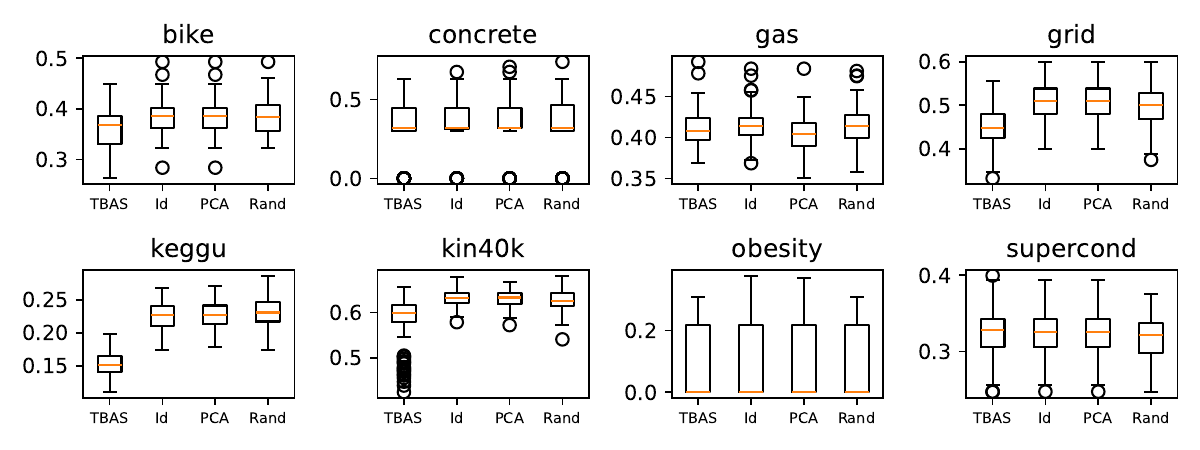}
	
	\textbf{Tree Depth 8}
	\includegraphics[width=0.95\textwidth]{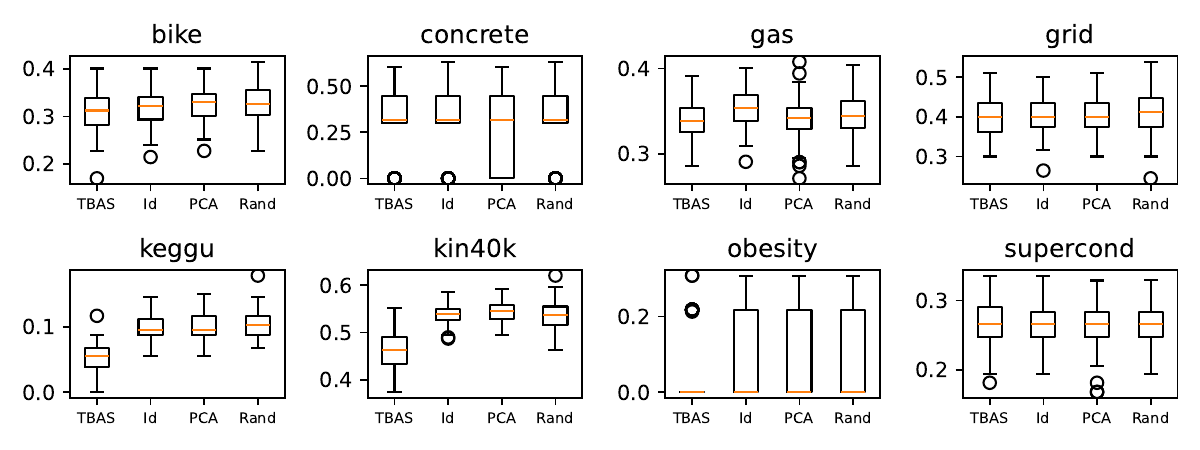}
	
	\textbf{Random Forest Depth 4}
	\includegraphics[width=0.95\textwidth]{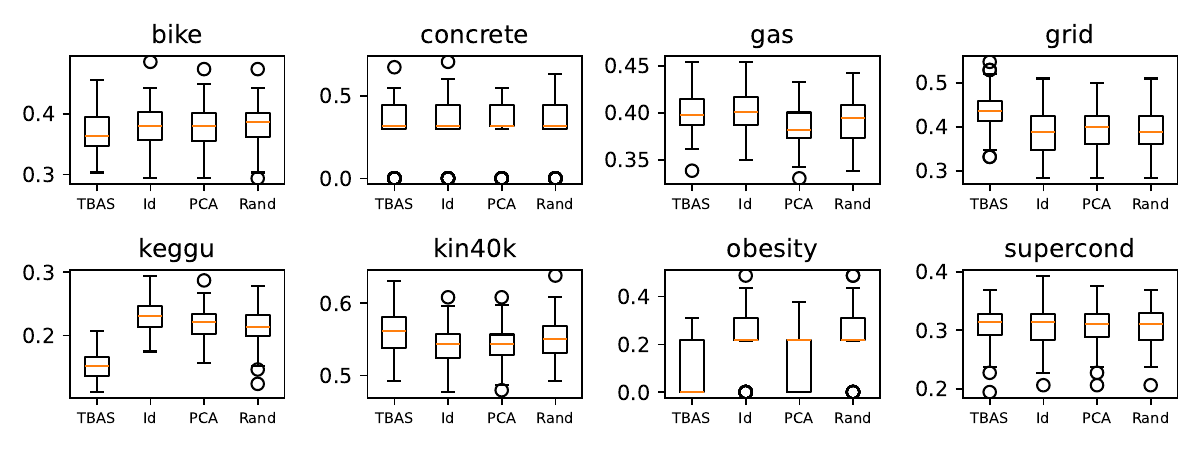}
	\caption{Classification Exercise: Brier Scores are indicated; lower is better.}
	\label{fig:class_bakeoff}
\end{figure}

\newpage
\section*{NeurIPS Paper Checklist}

\begin{enumerate}

\item {\bf Claims}
    \item[] Question: Do the main claims made in the abstract and introduction accurately reflect the paper's contributions and scope?
    \item[] Answer: \answerYes{} 
    \item[] Justification: We propose a new way to get gradient information from regression trees, as we mention in the introduction.
    \item[] Guidelines:
    \begin{itemize}
        \item The answer NA means that the abstract and introduction do not include the claims made in the paper.
        \item The abstract and/or introduction should clearly state the claims made, including the contributions made in the paper and important assumptions and limitations. A No or NA answer to this question will not be perceived well by the reviewers. 
        \item The claims made should match theoretical and experimental results, and reflect how much the results can be expected to generalize to other settings. 
        \item It is fine to include aspirational goals as motivation as long as it is clear that these goals are not attained by the paper. 
    \end{itemize}

\item {\bf Limitations}
    \item[] Question: Does the paper discuss the limitations of the work performed by the authors?
    \item[] Answer: \answerYes{} 
    \item[] Justification: Yes we have a dedicated Limitations section in the discussion.
    \item[] Guidelines:
    \begin{itemize}
        \item The answer NA means that the paper has no limitation while the answer No means that the paper has limitations, but those are not discussed in the paper. 
        \item The authors are encouraged to create a separate "Limitations" section in their paper.
        \item The paper should point out any strong assumptions and how robust the results are to violations of these assumptions (e.g., independence assumptions, noiseless settings, model well-specification, asymptotic approximations only holding locally). The authors should reflect on how these assumptions might be violated in practice and what the implications would be.
        \item The authors should reflect on the scope of the claims made, e.g., if the approach was only tested on a few datasets or with a few runs. In general, empirical results often depend on implicit assumptions, which should be articulated.
        \item The authors should reflect on the factors that influence the performance of the approach. For example, a facial recognition algorithm may perform poorly when image resolution is low or images are taken in low lighting. Or a speech-to-text system might not be used reliably to provide closed captions for online lectures because it fails to handle technical jargon.
        \item The authors should discuss the computational efficiency of the proposed algorithms and how they scale with dataset size.
        \item If applicable, the authors should discuss possible limitations of their approach to address problems of privacy and fairness.
        \item While the authors might fear that complete honesty about limitations might be used by reviewers as grounds for rejection, a worse outcome might be that reviewers discover limitations that aren't acknowledged in the paper. The authors should use their best judgment and recognize that individual actions in favor of transparency play an important role in developing norms that preserve the integrity of the community. Reviewers will be specifically instructed to not penalize honesty concerning limitations.
    \end{itemize}

\item {\bf Theory assumptions and proofs}
    \item[] Question: For each theoretical result, does the paper provide the full set of assumptions and a complete (and correct) proof?
    \item[] Answer: \answerYes{}
    \item[] Justification: We provide all proofs in the supplementary material, and carefully list our assumptions as part of Assumption 1 and Assumption 2 of the Appendix.
    \item[] Guidelines:
    \begin{itemize}
        \item The answer NA means that the paper does not include theoretical results. 
        \item All the theorems, formulas, and proofs in the paper should be numbered and cross-referenced.
        \item All assumptions should be clearly stated or referenced in the statement of any theorems.
        \item The proofs can either appear in the main paper or the supplemental material, but if they appear in the supplemental material, the authors are encouraged to provide a short proof sketch to provide intuition. 
        \item Inversely, any informal proof provided in the core of the paper should be complemented by formal proofs provided in appendix or supplemental material.
        \item Theorems and Lemmas that the proof relies upon should be properly referenced. 
    \end{itemize}

    \item {\bf Experimental result reproducibility}
    \item[] Question: Does the paper fully disclose all the information needed to reproduce the main experimental results of the paper to the extent that it affects the main claims and/or conclusions of the paper (regardless of whether the code and data are provided or not)?
    \item[] Answer: \answerYes{} 
    \item[] Justification: We provide all necessary details to reproduce our paper in the narrative and algorithms.
    \item[] Guidelines:
    \begin{itemize}
        \item The answer NA means that the paper does not include experiments.
        \item If the paper includes experiments, a No answer to this question will not be perceived well by the reviewers: Making the paper reproducible is important, regardless of whether the code and data are provided or not.
        \item If the contribution is a dataset and/or model, the authors should describe the steps taken to make their results reproducible or verifiable. 
        \item Depending on the contribution, reproducibility can be accomplished in various ways. For example, if the contribution is a novel architecture, describing the architecture fully might suffice, or if the contribution is a specific model and empirical evaluation, it may be necessary to either make it possible for others to replicate the model with the same dataset, or provide access to the model. In general. releasing code and data is often one good way to accomplish this, but reproducibility can also be provided via detailed instructions for how to replicate the results, access to a hosted model (e.g., in the case of a large language model), releasing of a model checkpoint, or other means that are appropriate to the research performed.
        \item While NeurIPS does not require releasing code, the conference does require all submissions to provide some reasonable avenue for reproducibility, which may depend on the nature of the contribution. For example
        \begin{enumerate}
            \item If the contribution is primarily a new algorithm, the paper should make it clear how to reproduce that algorithm.
            \item If the contribution is primarily a new model architecture, the paper should describe the architecture clearly and fully.
            \item If the contribution is a new model (e.g., a large language model), then there should either be a way to access this model for reproducing the results or a way to reproduce the model (e.g., with an open-source dataset or instructions for how to construct the dataset).
            \item We recognize that reproducibility may be tricky in some cases, in which case authors are welcome to describe the particular way they provide for reproducibility. In the case of closed-source models, it may be that access to the model is limited in some way (e.g., to registered users), but it should be possible for other researchers to have some path to reproducing or verifying the results.
        \end{enumerate}
    \end{itemize}

\item {\bf Open access to data and code}
    \item[] Question: Does the paper provide open access to the data and code, with sufficient instructions to faithfully reproduce the main experimental results, as described in supplemental material?
    \item[] Answer: \answerYes{} 
    \item[] Justification: We provide code reproducing our results and a README with instructions for doing so.
    \item[] Guidelines:
    \begin{itemize}
        \item The answer NA means that paper does not include experiments requiring code.
        \item Please see the NeurIPS code and data submission guidelines (\url{https://nips.cc/public/guides/CodeSubmissionPolicy}) for more details.
        \item While we encourage the release of code and data, we understand that this might not be possible, so “No” is an acceptable answer. Papers cannot be rejected simply for not including code, unless this is central to the contribution (e.g., for a new open-source benchmark).
        \item The instructions should contain the exact command and environment needed to run to reproduce the results. See the NeurIPS code and data submission guidelines (\url{https://nips.cc/public/guides/CodeSubmissionPolicy}) for more details.
        \item The authors should provide instructions on data access and preparation, including how to access the raw data, preprocessed data, intermediate data, and generated data, etc.
        \item The authors should provide scripts to reproduce all experimental results for the new proposed method and baselines. If only a subset of experiments are reproducible, they should state which ones are omitted from the script and why.
        \item At submission time, to preserve anonymity, the authors should release anonymized versions (if applicable).
        \item Providing as much information as possible in supplemental material (appended to the paper) is recommended, but including URLs to data and code is permitted.
    \end{itemize}

\item {\bf Experimental setting/details}
    \item[] Question: Does the paper specify all the training and test details (e.g., data splits, hyperparameters, how they were chosen, type of optimizer, etc.) necessary to understand the results?
    \item[] Answer: \answerYes{} 
    \item[] Justification: We provide all necessary information when describing our runs.
    \item[] Guidelines:
    \begin{itemize}
        \item The answer NA means that the paper does not include experiments.
        \item The experimental setting should be presented in the core of the paper to a level of detail that is necessary to appreciate the results and make sense of them.
        \item The full details can be provided either with the code, in appendix, or as supplemental material.
    \end{itemize}

\item {\bf Experiment statistical significance}
    \item[] Question: Does the paper report error bars suitably and correctly defined or other appropriate information about the statistical significance of the experiments?
    \item[] Answer: \answerYes{} 
    \item[] Justification: In our numerical results we use confidence interval overlap to assess significance.
    \item[] Guidelines:
    \begin{itemize}
        \item The answer NA means that the paper does not include experiments.
        \item The authors should answer "Yes" if the results are accompanied by error bars, confidence intervals, or statistical significance tests, at least for the experiments that support the main claims of the paper.
        \item The factors of variability that the error bars are capturing should be clearly stated (for example, train/test split, initialization, random drawing of some parameter, or overall run with given experimental conditions).
        \item The method for calculating the error bars should be explained (closed form formula, call to a library function, bootstrap, etc.)
        \item The assumptions made should be given (e.g., Normally distributed errors).
        \item It should be clear whether the error bar is the standard deviation or the standard error of the mean.
        \item It is OK to report 1-sigma error bars, but one should state it. The authors should preferably report a 2-sigma error bar than state that they have a 96\% CI, if the hypothesis of Normality of errors is not verified.
        \item For asymmetric distributions, the authors should be careful not to show in tables or figures symmetric error bars that would yield results that are out of range (e.g. negative error rates).
        \item If error bars are reported in tables or plots, The authors should explain in the text how they were calculated and reference the corresponding figures or tables in the text.
    \end{itemize}

\item {\bf Experiments compute resources}
    \item[] Question: For each experiment, does the paper provide sufficient information on the computer resources (type of compute workers, memory, time of execution) needed to reproduce the experiments?
    \item[] Answer: \answerYes{} 
    \item[] Justification: We provide this in Appendix B.2.
    \item[] Guidelines:
    \begin{itemize}
        \item The answer NA means that the paper does not include experiments.
        \item The paper should indicate the type of compute workers CPU or GPU, internal cluster, or cloud provider, including relevant memory and storage.
        \item The paper should provide the amount of compute required for each of the individual experimental runs as well as estimate the total compute. 
        \item The paper should disclose whether the full research project required more compute than the experiments reported in the paper (e.g., preliminary or failed experiments that didn't make it into the paper). 
    \end{itemize}
    
\item {\bf Code of ethics}
    \item[] Question: Does the research conducted in the paper conform, in every respect, with the NeurIPS Code of Ethics \url{https://neurips.cc/public/EthicsGuidelines}?
    \item[] Answer: \answerYes{} 
    \item[] Justification: Yes we respect this code of ethics.
    \item[] Guidelines:
    \begin{itemize}
        \item The answer NA means that the authors have not reviewed the NeurIPS Code of Ethics.
        \item If the authors answer No, they should explain the special circumstances that require a deviation from the Code of Ethics.
        \item The authors should make sure to preserve anonymity (e.g., if there is a special consideration due to laws or regulations in their jurisdiction).
    \end{itemize}

\item {\bf Broader impacts}
    \item[] Question: Does the paper discuss both potential positive societal impacts and negative societal impacts of the work performed?
    \item[] Answer: \answerYes{} 
    \item[] Justification: Yes we give discussion of this in the Discussion section.
    \item[] Guidelines:
    \begin{itemize}
        \item The answer NA means that there is no societal impact of the work performed.
        \item If the authors answer NA or No, they should explain why their work has no societal impact or why the paper does not address societal impact.
        \item Examples of negative societal impacts include potential malicious or unintended uses (e.g., disinformation, generating fake profiles, surveillance), fairness considerations (e.g., deployment of technologies that could make decisions that unfairly impact specific groups), privacy considerations, and security considerations.
        \item The conference expects that many papers will be foundational research and not tied to particular applications, let alone deployments. However, if there is a direct path to any negative applications, the authors should point it out. For example, it is legitimate to point out that an improvement in the quality of generative models could be used to generate deepfakes for disinformation. On the other hand, it is not needed to point out that a generic algorithm for optimizing neural networks could enable people to train models that generate Deepfakes faster.
        \item The authors should consider possible harms that could arise when the technology is being used as intended and functioning correctly, harms that could arise when the technology is being used as intended but gives incorrect results, and harms following from (intentional or unintentional) misuse of the technology.
        \item If there are negative societal impacts, the authors could also discuss possible mitigation strategies (e.g., gated release of models, providing defenses in addition to attacks, mechanisms for monitoring misuse, mechanisms to monitor how a system learns from feedback over time, improving the efficiency and accessibility of ML).
    \end{itemize}
    
\item {\bf Safeguards}
    \item[] Question: Does the paper describe safeguards that have been put in place for responsible release of data or models that have a high risk for misuse (e.g., pretrained language models, image generators, or scraped datasets)?
    \item[] Answer: \answerNA{} 
    \item[] Justification: This work has a low risk of misuse.
    \item[] Guidelines:
    \begin{itemize}
        \item The answer NA means that the paper poses no such risks.
        \item Released models that have a high risk for misuse or dual-use should be released with necessary safeguards to allow for controlled use of the model, for example by requiring that users adhere to usage guidelines or restrictions to access the model or implementing safety filters. 
        \item Datasets that have been scraped from the Internet could pose safety risks. The authors should describe how they avoided releasing unsafe images.
        \item We recognize that providing effective safeguards is challenging, and many papers do not require this, but we encourage authors to take this into account and make a best faith effort.
    \end{itemize}

\item {\bf Licenses for existing assets}
    \item[] Question: Are the creators or original owners of assets (e.g., code, data, models), used in the paper, properly credited and are the license and terms of use explicitly mentioned and properly respected?
    \item[] Answer: \answerYes{} 
    \item[] Justification: We give dataset information in Appendix B.1.
    \item[] Guidelines:
    \begin{itemize}
        \item The answer NA means that the paper does not use existing assets.
        \item The authors should cite the original paper that produced the code package or dataset.
        \item The authors should state which version of the asset is used and, if possible, include a URL.
        \item The name of the license (e.g., CC-BY 4.0) should be included for each asset.
        \item For scraped data from a particular source (e.g., website), the copyright and terms of service of that source should be provided.
        \item If assets are released, the license, copyright information, and terms of use in the package should be provided. For popular datasets, \url{paperswithcode.com/datasets} has curated licenses for some datasets. Their licensing guide can help determine the license of a dataset.
        \item For existing datasets that are re-packaged, both the original license and the license of the derived asset (if it has changed) should be provided.
        \item If this information is not available online, the authors are encouraged to reach out to the asset's creators.
    \end{itemize}

\item {\bf New assets}
    \item[] Question: Are new assets introduced in the paper well documented and is the documentation provided alongside the assets?
    \item[] Answer: \answerNA{} 
    \item[] Justification: This paper primarily introduces methodology.
    \item[] Guidelines:
    \begin{itemize}
        \item The answer NA means that the paper does not release new assets.
        \item Researchers should communicate the details of the dataset/code/model as part of their submissions via structured templates. This includes details about training, license, limitations, etc. 
        \item The paper should discuss whether and how consent was obtained from people whose asset is used.
        \item At submission time, remember to anonymize your assets (if applicable). You can either create an anonymized URL or include an anonymized zip file.
    \end{itemize}

\item {\bf Crowdsourcing and research with human subjects}
    \item[] Question: For crowdsourcing experiments and research with human subjects, does the paper include the full text of instructions given to participants and screenshots, if applicable, as well as details about compensation (if any)? 
    \item[] Answer: \answerNA{} 
    \item[] Justification: No crowdsourcing or other human subjects whatsoever.
    \item[] Guidelines:
    \begin{itemize}
        \item The answer NA means that the paper does not involve crowdsourcing nor research with human subjects.
        \item Including this information in the supplemental material is fine, but if the main contribution of the paper involves human subjects, then as much detail as possible should be included in the main paper. 
        \item According to the NeurIPS Code of Ethics, workers involved in data collection, curation, or other labor should be paid at least the minimum wage in the country of the data collector. 
    \end{itemize}

\item {\bf Institutional review board (IRB) approvals or equivalent for research with human subjects}
    \item[] Question: Does the paper describe potential risks incurred by study participants, whether such risks were disclosed to the subjects, and whether Institutional Review Board (IRB) approvals (or an equivalent approval/review based on the requirements of your country or institution) were obtained?
    \item[] Answer: \answerNA{} 
    \item[] Justification: No crowdsourcing or other human subjects whatsoever.
    \item[] Guidelines:
    \begin{itemize}
        \item The answer NA means that the paper does not involve crowdsourcing nor research with human subjects.
        \item Depending on the country in which research is conducted, IRB approval (or equivalent) may be required for any human subjects research. If you obtained IRB approval, you should clearly state this in the paper. 
        \item We recognize that the procedures for this may vary significantly between institutions and locations, and we expect authors to adhere to the NeurIPS Code of Ethics and the guidelines for their institution. 
        \item For initial submissions, do not include any information that would break anonymity (if applicable), such as the institution conducting the review.
    \end{itemize}

\item {\bf Declaration of LLM usage}
    \item[] Question: Does the paper describe the usage of LLMs if it is an important, original, or non-standard component of the core methods in this research? Note that if the LLM is used only for writing, editing, or formatting purposes and does not impact the core methodology, scientific rigorousness, or originality of the research, declaration is not required.
    \item[] Answer: \answerNA{} 
    \item[] Justification: LLMs were not involved in any aspect of this project.
    \item[] Guidelines:
    \begin{itemize}
        \item The answer NA means that the core method development in this research does not involve LLMs as any important, original, or non-standard components.
        \item Please refer to our LLM policy (\url{https://neurips.cc/Conferences/2025/LLM}) for what should or should not be described.
    \end{itemize}

\end{enumerate}

\end{document}